\documentclass{article}

\usepackage{microtype}
\usepackage{graphicx}
\usepackage{subfigure}
\usepackage{booktabs} % for professional tables

\usepackage{hyperref}

\usepackage[accepted]{icml2024}

\usepackage{amsmath}
\usepackage{amssymb}
\usepackage{mathtools}
\usepackage{amsthm}

\usepackage[capitalize,noabbrev]{cleveref}

\theoremstyle{plain}

\newtheorem{proposition}{Proposition}

\theoremstyle{definition}

\theoremstyle{remark}

\usepackage[textsize=tiny]{todonotes}
\usepackage{amsfonts}
\usepackage{adjustbox}
\usepackage{multirow}
\usepackage{subfigure}
\usepackage{colortbl}
\usepackage{cuted}
\usepackage{pifont}
\newcommand{\cmark}{\ding{51}}%
\newcommand{\xmark}{\ding{55}}%

\def\eg{\emph{e.g}.}
\def\ie{\emph{i.e}.}

\def\normal{\mathcal{N}(\mathbf{0}, \mathbf{I})}
\def\N{\mathcal{N}}
\def\I{\mathbf{I}}
\def\R{\mathbb{R}}
\def\Pr{\mathbf{Pr}}

\def\xb{\mathbf{x}}
\def\yb{\mathbf{y}}
\def\xo{\boldsymbol{z}}
\def\xot{\boldsymbol{z}_t}
\def\xott{\boldsymbol{z}_{t-1}}
\def\xos{\boldsymbol{z}_0}
\def\xstil{\hat{\xb}_0}

\def\xt{\mathbf{x}_t}
\def\xtt{\mathbf{x}_{t-1}}
\def\yt{\mathbf{y}_{t}}
\def\ytt{\mathbf{y}_{t-1}}
\def\xT{\mathbf{x}_T}

\def\Q{\boldsymbol{Q}}

\def\xs{\mathbf{x}_0}
\def\ys{\mathbf{y}_0}
\def\xf{\mathbf{x}_1}
\def\yf{\mathbf{y}_1}
\def\xsT{\mathbf{x}_{0:T}}

\def\xfT{\mathbf{x}_{1:T}}
\def\yfT{\mathbf{y}_{1:T}}

\def\EE{\mathbb{E}}

\def\ll{\left}
\def\rr{\right}

\def\ptheta{p_\theta}

\def\ll{\left[}
\def\rr{\right]}
\icmltitlerunning{Stochastic Conditional Diffusion Models for Robust Semantic Image Synthesis}

\begin{document}

\twocolumn[
\icmltitle{Stochastic Conditional Diffusion Models for Robust Semantic Image Synthesis}
\icmlsetsymbol{equal}{*}

\begin{icmlauthorlist}
\icmlauthor{Juyeon Ko}{equal,yyy}
\icmlauthor{Inho Kong}{equal,yyy}
\icmlauthor{Dogyun Park}{yyy}
\icmlauthor{Hyunwoo J. Kim}{yyy}
\end{icmlauthorlist}

\icmlaffiliation{yyy}{Department of Computer Science, Korea University, Republic of Korea}

\icmlcorrespondingauthor{Hyunwoo J. Kim}{hyunwoojkim@korea.ac.kr}

\icmlkeywords{diffusion models, semantic image synthesis, conditional generation}

\vskip 0.3in
]



\printAffiliationsAndNotice{\icmlEqualContribution} % otherwise use the standard text.

\begin{abstract}
Semantic image synthesis (SIS) is a task to generate realistic images corresponding to semantic maps (labels).
However, in real-world applications, SIS often encounters noisy user inputs. 
To address this, we propose Stochastic Conditional Diffusion Model (SCDM), which is a robust conditional diffusion model that features novel forward and generation processes tailored for SIS with noisy labels.
It enhances robustness by stochastically perturbing the semantic label maps through Label Diffusion, which diffuses the labels with discrete diffusion.
Through the diffusion of labels, the noisy and clean semantic maps become similar as the timestep increases, eventually becoming identical at $t=T$.
This facilitates the generation of an image close to a clean image, enabling robust generation.
Furthermore, we propose a class-wise noise schedule to differentially diffuse the labels depending on the class.
We demonstrate that the proposed method generates high-quality samples through extensive experiments and analyses on benchmark datasets, including a novel experimental setup simulating human errors during real-world applications.
Code is available at \href{https://github.com/mlvlab/SCDM}{https://github.com/mlvlab/SCDM}.
\end{abstract}
\section{Introduction}
\begin{figure*}[t!]
\centering
\includegraphics[width=1.0\textwidth]{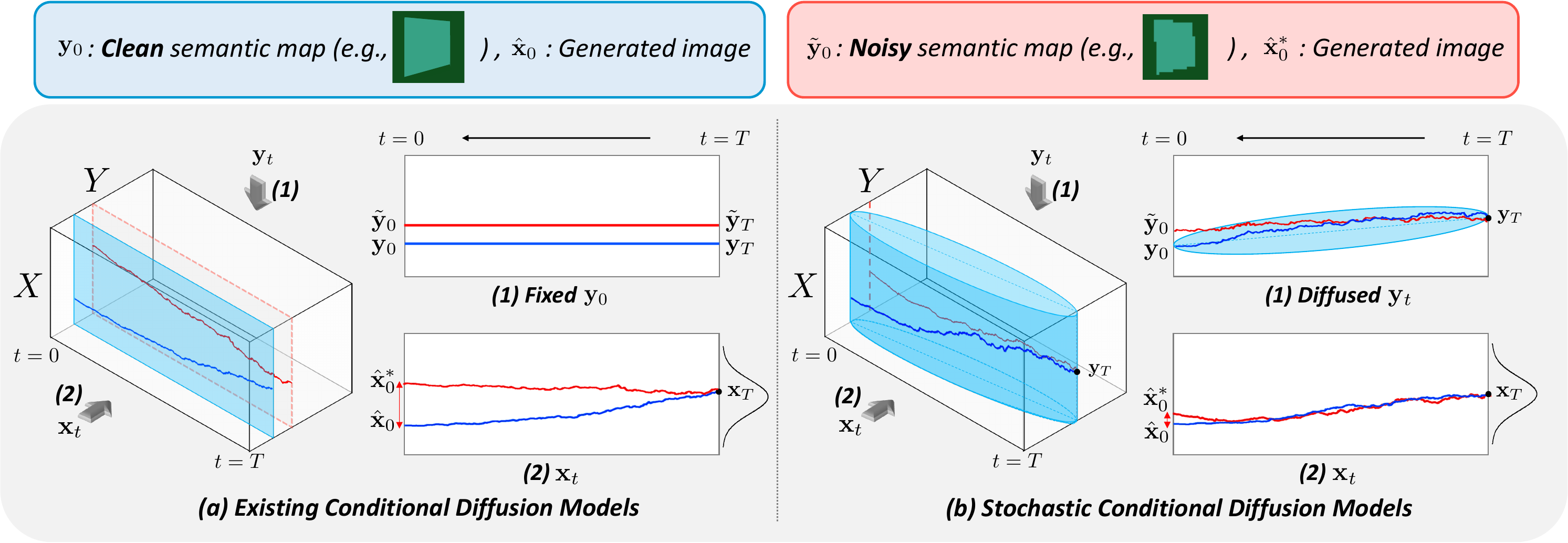}
\caption{\textbf{Visualization of conditional generation.}
Each colored trajectory represents a sampling trajectory conditioned on a noisy semantic map $\tilde{\yb}_0$~\textcolor{red}{(\textbf{Red})} and the corresponding clean semantic map $\ys$~\textcolor{blue}{(\textbf{Blue})}.
They are projected onto the \textbf{(1)} semantic map space and the \textbf{(2)} image space, sharing the same $\xb_T$.
\textbf{(a)} Existing conditional diffusion models (baseline) use a fixed condition $\ys$ over the generation process, and the gap between $\tilde{\yb}_t$ and $\yt$ yields erroneous conditional score estimation at each timestep $t$.
\textbf{(b)} In contrast, our method stochastically perturbs the condition with masking, resulting in a trajectory $\yb_{1:T}$ following a probability distribution $q(\yb_{1:T}|\yb_0)$, as depicted with blue shaded areas around the $\yt$ trajectory.
This makes the intermediate trajectories, \ie, $\yb_{1:T}|\yb_0$ and $\tilde{\yb}_{1:T}|\tilde{\yb}_0$, close to each other, enhancing the robustness against the noisy labels.
}
\label{fig:motivation2}
\end{figure*}
Semantic image synthesis (SIS) is a type of image translation that converts a given semantic map (label) into a photo-realistic image, which is the inverse of semantic segmentation.
It is also one of the conditional image generation tasks with semantic label maps serving as the input conditions. 
The problem is formulated as approximating the conditional distribution $q(X|Y)$ where $X$ and $Y$ are the random variables denoting the image and the semantic map, respectively.
SIS is addressed by adopting conditional generative models~\cite{yang2019diversity,tang2020dual,DBLP:conf/eccv/NtavelisRKGT20,tan2021diverse} such as conditional GANs~\cite{DBLP:conf/nips/GoodfellowPMXWOCB14,mirza2014conditional}.
Given a semantic label map $\mathbf{y}$, these works sample a new image $\hat{\mathbf{x}}$ from the learned conditional distribution $p_\theta(X|Y=\mathbf{y})$.
As diffusion models~\cite{sohl2015deep,ho2020denoising,song2020score} have gained significant attention on various generation tasks~\cite{ramesh2021zero,dhariwal2021diffusion,CouaironVSC23}, diffusion models for SIS have been recently studied by a few works~\cite{wang2022pretraining, xue2023freestyle}.
Specifically, SDM~\cite{wang2022semantic} embeds the input conditions similarly to SPADE~\cite{park2019semantic} and integrates diffusion models into the context of SIS.
LDM~\cite{rombach2022high} learns a diffusion model on latent vectors with condition encoders.

SIS has a wide range of real-world applications such as photo editing or content creation~\cite{chen2017photographic,park2019semantic,zhu2020sean,tang2020local}.
In practice, SIS often involves noisy input $\tilde{\mathbf{y}}$ from users.
For instance, users mark specific areas as the classes they wish to synthesize, and the marks come with errors such as jagged edges and incompletely marked areas.
Even labels by professional annotators in benchmark datasets occasionally contain mistakes and show inconsistency between annotators, and the input from end users would inevitably entail noise.
This poses the gap between the label distributions for training and inference. 
Usually, models are trained with clean labels $\yb$ in benchmark datasets, whereas generation is performed with noisy labels $\tilde{\yb}$.
In the case of diffusion models, the model is exposed to the erroneous guidance throughout $T$ timesteps, \ie, $t=T$ to $t=1$, generating the corresponding noisy image.

To minimize the gap, we generate samples with stochastically perturbed labels for both training and inference.
Specifically, we propose to diffuse the semantic label map $\ys$ to $\yb_1$, $\ldots$, $\yb_T$ and use them throughout the generation process.
Assume there exists a clean semantic map $\yb$ corresponding to the noisy one $\tilde{\yb}$.
Then, utilizing discrete diffusion with an absorbing state allows us to make the intermediate noisy map $\tilde{\yb}_t$ and clean map $\yb_t$ gradually become similar, as they are masked and eventually become identical at $t=T$.
Since the trajectories $\yb_{1:T}$ and $\tilde{\yb}_{1:T}$ provided to the model during the generation process are similar, \ie, $\yb_t$ and $\tilde{\yb}_t$ get closer than $\yb$ and $\tilde{\yb}$, the generated image is close to the clean image, as illustrated in Figure~\ref{fig:motivation2}.

In this paper, we introduce \textbf{Stochastic Conditional Diffusion Model~(SCDM)}, a novel conditional diffusion model specifically designed to enhance robustness on noisy labels.
Our SCDM stochastically perturbs the semantic maps with \textit{Label Diffusion} and conditions image generation on the \textit{diffused} labels.
We also incorporate label statistics and develop a new class-wise noise schedule for labels to enhance the generation quality of small and rare classes.
Moreover, the generation process of SCDM entails two heterogeneous diffusion processes: a \textit{discrete forward process} for labels and a \textit{continuous reverse process} for images.
We empirically demonstrate that SCDM can approximate $q(X|Y)$ and present theoretical analysis.
Additionally, we introduce a new noisy SIS benchmark and prove the robustness of SCDM under noisy conditions.

To summarize, our \textbf{contributions} are as follows:
\begin{itemize}
    \item We propose Stochastic Conditional Diffusion Model (SCDM), a robust conditional diffusion model for SIS that incorporates Label Diffusion, a discrete diffusion process for labels that enables differential conditioning on semantic labels.
    \item We provide theoretical analyses of our class-wise schedule and the relationship between the class guidances (implicit classifier gradients) induced by fixed labels and label diffusion.
    \item We introduce a new SIS benchmark designed to assess generation performance under noisy conditions, simulating human errors that can occur during real-world applications. 
    \item We conduct extensive experiments and analyses on benchmark datasets and achieve competitive results.
\end{itemize}
\section{Related Works}
\paragraph{Semantic Image Synthesis.}
Since Pix2pix~\cite{isola2017image} have established a general framework for SIS, conditional Generative Adversarial Networks (GANs) are widely used in SIS~\cite{yang2019diversity,zhu2017toward,DBLP:conf/eccv/NtavelisRKGT20,tang2020local,tan2021diverse,tan2021efficient,shi2022retrieval}.
SPADE~\cite{park2019semantic} proposes spatially-adaptive normalization and successfully preserves semantic information.
Since SPADE, many normalization-based approaches have been presented~\cite{tan2021diverse,lv2022semantic}.
For instance, 
CLADE~\cite{tan2021efficient} adopts class-adaptive normalization and 
RESAIL~\cite{shi2022retrieval} proposes retrieval-based spatially adaptive normalization, and OASIS~\cite{sushko2020you} designs the discriminator as a semantic segmentation network. 
INADE~\cite{tan2021diverse} utilizes class-level conditional modulation, and SAFM~\cite{lv2022semantic} proposes shape-aware position descriptors to modulate the features.

Recently, diffusion models (DMs) also have been proposed for SIS.
SDM~\cite{wang2022semantic} encodes the semantic label map with SPADE.
LDM~\cite{rombach2022high} leverages a latent space for the conditions including the semantic maps.
PITI~\cite{wang2022pretraining} pre-trains the semantic latent space and finetunes it with the RGB-preprocessed semantic mask images, rather than using the maps with class indexes like most studies including ours.
FLIS~\cite{xue2023freestyle} incorporates not only semantic label maps but also additional text inputs. 
Most of these works have recognized the applicability of SIS in real-world scenarios.
However, to the best of our knowledge, our method is the first DM-based model to address the issue of noisy user inputs in SIS.

\paragraph{Conditional Diffusion Models.}
By modifying the U-Net \cite{ronneberger2015u} architecture to incorporate the conditions into the network, conditional diffusion models are prevalently leveraged for conditional generations.
ADM~\cite{dhariwal2021diffusion}, for example, encodes the class embedding into the network with AdaGN and utilizes classifier guidance. 
LDM~\cite{rombach2022high} extracts features from the various conditions and encodes them into the network through concatenation or cross-attention.
DDMI~\cite{park2024ddmi} also generates data with latent diffusion but in a domain-agnostic manner, by linking data to a continuous function within a shared latent space.
Meanwhile, SDM~\cite{wang2022semantic} replaces GroupNorm of U-Net decoder with SPADE to embed the labels into the network in a spatially adaptive manner.
UNIT-DDPM~\cite{sasaki2021unit} uses two different diffusion models and a domain translation function for the training and sampling of an image-to-image translation model.
Recently, T2I diffusion models such as FLIS~\cite{xue2023freestyle}, ControlNet~\cite{zhang2023adding}, and GLIGEN~\cite{li2023gligen} have paved the way for adding spatial controls to large pretrained diffusion models.
Our method diffuses labels with a carefully designed discrete diffusion process and generates images conditioned on the diffused labels, formulating a novel conditional diffusion model.
\section{Preliminaries}
In this work, we consider conditional diffusion model (DM) $p_\theta(\xs|\ys)$ for SIS task.
We briefly introduce existing DM-based models that learn the conditional distribution $q(\xs|\ys)$ by estimating a reverse process given a fixed label $\ys$ that approximates an unconditional forward process $q(\xb_{1:T}|\xs)$.
Then, as our method perturbs labels $\ys$ using a discrete diffusion, we summarize basic concepts of a discrete state space diffusion process.
\vspace{-10pt}
\paragraph{Semantic Image Synthesis with Diffusion Models.}
\label{sec:background}
In previous methods, the forward process is defined as an unconditional diffusion model that adds Gaussian noise to the image as follows:
\vspace{-8pt}
\begin{equation}
\label{eq:img_diff_onestep}
   q(\xt|\xtt) := \N \left (\xt; \sqrt{\frac{\alpha_t}{\alpha_{t-1}}} \xtt, \left (1-\frac{\alpha_t}{\alpha_{t-1}} \right ) \I \right ),
\end{equation}
where the decreasing sequence $\alpha_{1:T}$ defines the noise level with strictly positive $\alpha_t$.
For the forward process, 
we have a closed-form sampling step of $\xt$ at an arbitrary timestep $t$, $q(\xt|\xs) = \N(\xt; \sqrt{\alpha_t} \xs, (1-\alpha_t) \I)$, as it is defined as a Markov chain.
The reverse process $p_\theta(\xb_{0:T}|\ys)$ is also defined as a Markov chain with learned transitions starting from $p(\xb_T)=\N(\xb_T; \mathbf{0}, \I)$ given as:
\vspace{-8pt}
\begin{equation}
    \ptheta(\xsT|\ys) := p(\xT)\prod^T_{t=1}\ptheta(\xtt|\xt, \ys).
\end{equation}
It is usually learned by a deep neural network parameterized by $\theta$ that represents Gaussian transitions given as:
\begin{equation}
    p_\theta(\xtt|\xt,\ys) := \N(\xtt; \boldsymbol{\mu}_\theta(\xt, \ys, t), \mathbf{\Sigma}_\theta(\xt, \ys, t)).
\end{equation}
The conditional DM can be trained with the hybrid loss from \cite{nichol2021improved} as:
\begin{align}
    \mathcal{L}_\text{hybrid} &= \mathcal{L}_\text{simple} + \lambda \mathcal{L}_\text{vlb}, \\
    \mathcal{L}_\text{simple} &= \EE_{t,\xs, \ys, \epsilon} \ll||\epsilon - \epsilon_\theta(\sqrt{\alpha_t}\xs + \sqrt{1-\alpha_t}\epsilon, \ys, t)||_2^2\rr, \nonumber 
 \\
    \mathcal{L}_\text{vlb} &= D_\text{KL}\left(p_\theta(\xtt|\xt,\ys)||q(\xtt|\xt,\xs)\right), \nonumber 
\end{align}
where $\lambda$ is a balancing hyperparamer and $\epsilon_\theta$ is a noise prediction model.
\vspace{-8pt}
\paragraph{Discrete State Space Diffusion Process.}
For a discrete categorical random variable $z \in \{ 1, ..., C\}$ with $C$ categories, DMs for discrete state spaces~\cite{hoogeboom2021argmax, austin2021structured} are defined with transition matrices where $[\Q_t]_{ij} = q(z_t=i|z_{t-1}=j)$ and $\Q_t \in \R^{C\times C}$.
The forward process is then defined as follows:
\begin{equation}
\label{eq:segmap_diff_onestep}
    q(\xot|\xott) := \text{Cat}(\xot; \boldsymbol{p}= \Q_t \xott),
\end{equation}
where $\xo$ is the one-hot column vector~($\mathbf{e}_z$) and $\text{Cat}(\xo;\boldsymbol{p})$ is a categorical distribution parameterized by $\boldsymbol{p}$.
We can sample $\xot$ at an arbitrary timestep $t$ from the following marginal starting from $\xos$ in closed form, due to the Markov property:
\begin{equation}
\begin{gathered}
\label{eq:segmap_marginal}
    q(\xot|\xos)  := \text{Cat}(\xot; \boldsymbol{p} = \overline{\Q}_t \xos), \\
    \text{where } \overline{\Q}_t = \Q_t \Q_{t-1} ... \Q_1.
\end{gathered}
\end{equation}

As the design choice of the transition matrix $\Q_t$ determines the diffusion process, one has to choose the matrix carefully.
For instance, D3PM~\cite{austin2021structured} controlled the data corruption by designing the matrix with domain knowledge or structure such as text token embedding distance.
\section{Stochastic Conditional Diffusion Model}
\begin{figure*}[!ht]
\includegraphics[width=1.0\textwidth]{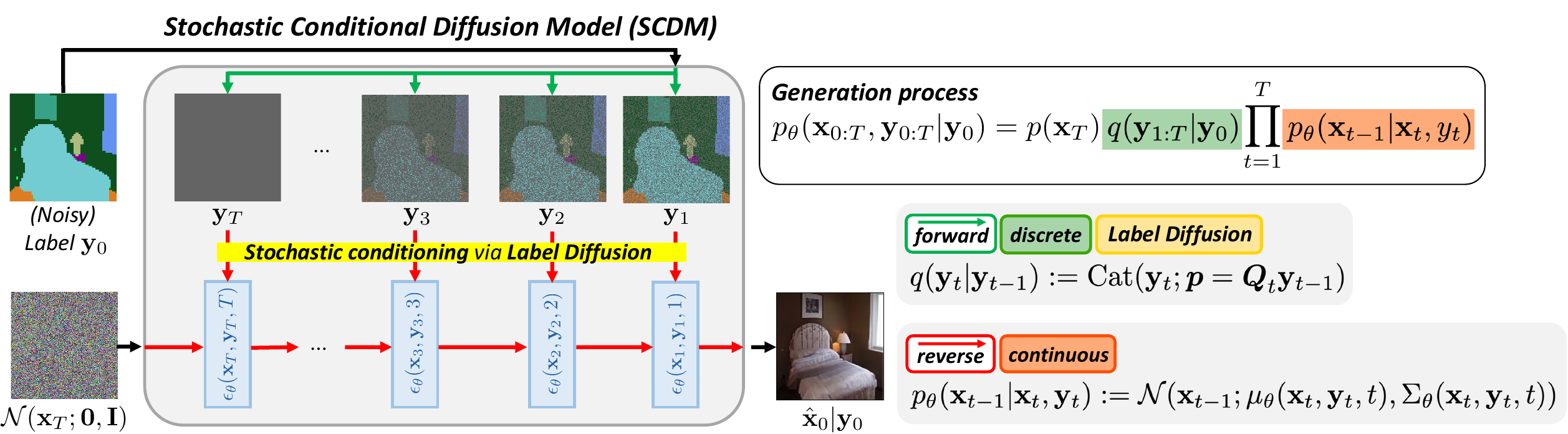}
\caption{
\textbf{Generation process of SCDM.} 
The Stochastic Conditional Diffusion Model~(SCDM) is a robust conditional diffusion model for semantic image synthesis.
SCDM consists of a \textit{discrete forward process} for labels and a \textit{continuous reverse process} for images.
It improves the robustness to noisy semantic labels as well as generation performance on clean semantic labels.
$\xos$ denotes the $i$-th pixel of the semantic map, \ie, $\xos = \ys^{i}$ where $\ys = \{ \ys^{1},...,\ys^{H \times W} \}$.
}
\label{fig:main}
\end{figure*}
We propose Stochastic Conditional Diffusion Model (SCDM), a robust and novel conditional diffusion model for semantic image synthesis.
In this section, we introduce our forward and generation processes of SCDM (Section~\ref{sec:label-diffusion-model}), discrete label diffusion process (Section~\ref{sec:transition-matrix-and-noise-scheduling}), and the training and sampling schemes of our method (Section~\ref{sec:training-sampling}).
\subsection{Definitions}
\label{sec:label-diffusion-model}
Stochastic Conditional Diffusion Model (SCDM) is a class of conditional diffusion models approximating the conditional distribution $q(\xs|\ys).$
It conditions on diffused labels, \ie, $\yb_{1:T}$, given $\ys$.

Our SCDM is defined as follows:
\begin{equation}
    \ptheta(\xs|\ys) := \int\int\ptheta(\xsT,\yfT|\ys)\mathrm{d\xb}_{1:T}\mathrm{d\yb}_{1:T},
\end{equation}
where $\xb_{1:T}$ and $\yb_{1:T}$ are latents with the same dimensionality as $\xs$ and $\ys$ respectively, and $(\xs,\ys) \sim q(\xs,\ys)$.\\

\noindent\textbf{Forward process.} 
SCDM consists of two diffusion processes: a \textit{continuous} diffusion process $q(\xt|\xtt)$ for images as in Eq.~\eqref{eq:img_diff_onestep} and a \textit{discrete} diffusion process $q(\yt|\ytt)$ for categorical semantic labels as in Eq.~\eqref{eq:segmap_diff_onestep}.
We name the discrete diffusion process Label Diffusion.
Then, the forward process of SCDM is defined as follows:
\begin{align}
    q(\xfT,\yfT|\xs,\ys) &:= \prod^T_{t=1}q(\xt,\yt|\xtt, \ytt), \\
    \label{eq:forward}
    q(\xt,\yt|\xtt,\ytt) &:= q(\xt|\xtt)q(\yt|\ytt).
\end{align}
For simplicity, we employed the same $t$ for both the image $\xb$ and the label $\yb$, but the noise levels are separately controlled.
We manually designed a different noise schedule for the labels, referred to as the class-wise noise schedule, which will be discussed in the following section.
Consequently, the diffusion processes for $\xb$ and $\yb$ are not synchronized regarding noise levels, and the synchronization between them is not necessary.

Note that although we diffuse images and labels independently, 
$\xt$ and $\yt$ are still correlated as $\xs$ and $\ys$ are dependent (\ie, $\ys$ is deterministically decided given $\xs$).

\noindent\textbf{Generation process.} We define the joint distribution for our generation process as:
\begin{align}
    \ptheta(\xsT,\yfT|\ys) &:= p(\xT)q(\yfT|\ys)\prod^T_{t=1}\ptheta(\xtt|\xt, \yt), \\
    \ptheta(\xtt|\xt,\yt) &:= \N(\xtt; \mathbf{\mu}_\theta(\xt,\yt,t), \mathbf{\Sigma}_\theta(\xt,\yt,t)),
\end{align}
where $q(\yfT|\ys)$ is the discrete Label Diffusion forward process and $\ptheta(\xtt|\xt,\yt)$ is the continuous reverse process.
As $\ys$ is given in SIS, we only have to sample the image $\xstil$ and we do not need a reverse process for the label.
Therefore, we define our generation process with the forward process $q(\yfT|\ys)$ and the intermediate $\yfT$ are obtained from $\ys$ without any neural network evaluations.
\subsection{Label Diffusion}
\label{sec:transition-matrix-and-noise-scheduling}
We introduce our Label Diffusion, a new discrete diffusion with \textit{label masking} and \textit{class-wise noise scheduling}.

\subsubsection{Transition Matrix for Label Masking}
To gradually erase the information of the semantic map and increase similarity among different maps as $t=0 \rightarrow t=T$, we designed our Label Diffusion process by progressively masking labels.
In other words, the original semantic labels are converted into the absorbing state (\ie, $\texttt{[mask]}$) with some probability at each timestep.
Consequently, all semantic maps eventually become identical at $t=T$, each filled with $\texttt{[mask]}$ at every pixel.
In addition, as the semantic labels are discrete variables, \ie, classes, it is natural to adopt the discrete diffusion.

Given semantic classes $C$ and absorbing state $\texttt{[mask]}$, 
we define the transition matrix $\Q_t \in \R^{(C+1)\times(C+1)}$ at a timestep $t$ as:
\vspace{-2mm}
\begin{equation}
    \left[\Q_{t}\right]_{i j}= \begin{cases} 1-\beta_{t,c} & \text { if } \quad i=j=c, \\ \beta_{t,c} & \text { if } \quad i=C+1, j=c, \\ 1 & \text { if } \quad i=j= C+1, \\ 0 & \text{ otherwise,}\end{cases} 
\end{equation}
where the absorbing state is added as $(C+1)$-th class and $\beta_{t,c}$ is the probability of a label of class $c$ to be masked. 
Note that class-wise defined probability $\beta_{t,c}$ enables class-wise \textit{noise scheduling}, which will be further discussed in the following section.
We assume that $\boldsymbol{Q}_t$ is applied to each pixel of the semantic map independently.
Since $\texttt{[mask]}$ is the absorbing state, once the state is in $C+1$ then the next state is always $C+1$.

In addition, the probability $q(\xot|\xos)$ at $t$ starting from $\xos$, which is one-hot vector with 1 at $c$-th entry as in Eq.~\eqref{eq:segmap_marginal}, is:
\vspace{-1mm}
\begin{equation}
    \label{eq:marginal}
    q(\xot|\xos)  = \text{Cat}\left(\xot; \boldsymbol{p} = \overline{\Q}_t \xos = (1-\gamma_{t,c})\xos+ \gamma_{t,c} \mathbf{e}_{C+1}\right),
    % \overline{\Q}_t \xos = (1-\gamma_{t,c})\xos+ \gamma_{t,c} \mathbf{e}_{C+1},
\end{equation}
where $\gamma_{t,c} = 1-\prod^t_{i=1} (1-\beta_{i,c})$ denotes the probability that a semantic label $c$ has been assigned to the absorbing state until timestep $t$ and $\mathbf{e}_{C+1}$ is the one-hot column vector where $(C+1)$-th entry is 1.
Utilizing $\gamma_{t,c}$ not only simplifies the implementation of the transition kernel but also reduces memory consumption in the generation process.
The entire trajectory $\yb_{1:T}|\ys$ can be efficiently represented with a single $\mathbb{R}^{H\times W}$ matrix.
For more details, see Appendix~\ref{sec:supp-trajectory}.

\subsubsection{Noise Scheduling} 
We observe that the semantic information of small objects in an image is prone to be lost at a relatively early stage of diffusion compared to larger objects. 
Moreover, for a rare class of objects in the dataset, it would be hard to learn their semantics if their labels are masked quickly.
Thus, we propose a \textbf{class-wise noise schedule} to differentially transform semantic labels depending on the class.
We designed $\gamma_{t,c}$ to ensure that labels occupying smaller areas and rarely appearing in the dataset are transitioned into the absorbing state more slowly and at a later time.
The class-wise schedule improves generation quality for small and rare objects. 
For more details, see Section~\ref{sec:noise_schedule_analysis} and Appendix~\ref{sec:supp-classwise}.

For a given class $c$, we introduce $\psi_c$ as defined in Eq.~\eqref{eq:itf} and $\phi_c$ in Eq.~\eqref{eq:idf}.
These terms take into account the area ($\approx$ object size) and frequency of the class $c$, respectively, for noise scheduling.
We estimated $\psi$ and $\phi$ with training data.
Using $\psi$ and $\phi$, 
we define $\gamma_{t,c}$ for the class-wise noise schedule as:
\begin{gather}
\label{eq:gamma_t}
    \gamma_{t,c} := \frac{(\psi_c \phi_c)^{\eta \frac{t}{T}}-1}{(\psi_c \phi_c)^{\eta} -1}, \\
\label{eq:itf}
    \psi_c = \EE_{\xb \in \mathcal{X}_c} \Big[ \Pr(y_{ij}=c|\xb) \Big]^{-1}, \\
\label{eq:idf}
    \phi_c = \log \Big( \Pr(\xb\in \mathcal{X}_c)^{-1} \Big),
\end{gather}

where $\mathcal{X}_c$ is the set of images containing class $c$, $y_{ij}$ is the class label of the semantic map $\yb$ at $(i,j)$,
and $\eta$ is a hyperparameter.
This properly slows down the label diffusion of small and rare objects.
We provide a visual aid for $\gamma_{t,c}$ with different $\psi_c\phi_c$ values in Appendix~\ref{sec:supp-psiphi}.

\begin{proposition}
\label{lemma1}
For $\gamma_{t,c}$ in Eq.~\eqref{eq:gamma_t} with $\psi_c\phi_c > 1$ for all $c$ and $t<T$,
\begin{equation*}
    \lim_{\eta \rightarrow 0} \gamma_{t,c} = \frac{t}{T} \;\;\;\text{ and}\;\;\lim_{\eta \rightarrow \infty} \gamma_{t,c} = 0.
\end{equation*}
\end{proposition}
Interestingly, our class-wise schedule generalizes the linear and uniform schedule, and no Label Diffusion. 
As $\eta \rightarrow 0$, the class-wise schedule defined in \eqref{eq:gamma_t} converges to a linear and uniform noise schedule, \ie, $\frac{t}{T}$.
Labels across all classes have the same probability to be masked, and the marginal probability \eqref{eq:marginal} linearly increases. 
This schedule is the same as the one leveraged in absorbing-state D3PM~\cite{austin2021structured}.
Also, as $\eta \rightarrow \infty$, the masking probability approaches zero.
Formally, this property is summarized in Proposition \ref{lemma1}, and its proof is provided in Appendix~\ref{sec:supp-lemma1-proof}.
\subsection{Training and Sampling}
\label{sec:training-sampling}
We train our network with the following loss function:
\begin{equation}
\label{eq:loss}
    \mathcal{L} = \mathcal{L}_\text{simple} + \lambda \mathcal{L}_\text{vlb},
\end{equation}
\begin{equation}
    \mathcal{L}_\text{simple} = \EE_{t,\xs, \yt, \epsilon} \ll||\epsilon - \epsilon_\theta(\sqrt{\alpha_t}\xs + \sqrt{1-\alpha_t}\epsilon, \yt, t)||_2^2\rr,
\end{equation}
\begin{equation}
    \mathcal{L}_\text{vlb} = D_\text{KL}\left(p_\theta(\xtt|\xt,\yt)||q(\xtt|\xt,\xs)\right),
\end{equation}
where $\alpha_t$ determines the noise level for input image $\xb$ at timestep $t$, $\epsilon \sim \normal$ is a Gaussian noise, and $\lambda$ is a balancing hyperparameter.
This is similar to the hybrid loss~\cite{nichol2021improved}, with a slight adaptation of using $\yt$.
See Appendix~\ref{sec:supp-vlb-proof} for a detailed derivation of the objective function.
Note that Label Diffusion does not significantly impact the training cost of the main diffusion.

To generate a sample, by the definition of the generation process, we first sample $\yfT$ from the semantic map $\ys$ and feed them sequentially to the process.
In addition, we formulate the classifier-free guidance~\cite{ho2022classifier} in our model:
\begin{equation}
\label{eq:cfg}
    \tilde{\epsilon}_\theta(\xt|\yt) = \epsilon_\theta(\xt|\yt) + s(\epsilon_\theta(\xt|\yt) - \epsilon_\theta(\xt|\emptyset)),
\end{equation}
where $s$ is the guidance scale.
This is similar to a line of works~\cite{nichol2022glide,rombach2022high,wang2022semantic} adopting the guidance, but we use the perturbed labels $\yt$ by our discrete forward process instead of the clean and fixed label $\ys$ for all steps. 

\paragraph{Extrapolation.}
Inspired by \cite{lu2022dpm}, we give additional guidance in $\xs$ space, as opposed to directly sampling $\mathbf{x}_{t-1}$.
We first compute $\xs^{(t)}$ through $\tilde{\epsilon}_\theta(\xt|\yt)$.
This value is then extrapolated from the preceding time-step prediction, $\tilde{\mathbf x}_0^{(t+1)},$ using the formula:
\begin{equation}
\label{eq:extrapolation}
    \tilde{\mathbf x}_0^{(t)} = \xs^{(t)} + w \left (\xs^{(t)} - \tilde{\mathbf x}_0^{(t+1)} \right ),
\end{equation}
where the constant extrapolation scale is denoted by $w$.
Following extrapolation, we apply dynamic thresholding~\cite{saharia2022photorealistic} and subsequently randomly sample $\mathbf{x}_{t-1}$ utilizing $\xt$ and $\tilde{\mathbf x}_0^{(t)}$.

The complete training and sampling algorithm is in Appendix~\ref{sec:supp-B}.
Note that the overall algorithms remain mostly unchanged and the only modification involves incorporating $\texttt{[mask]}$ into labels through Label Diffusion.
This introduces minimal computational overhead.
\subsection{Discussion}
\label{sec:discussion}
We here theoretically analyze the generative processes of SCDM to discuss an interesting relationship with the \emph{fixed} conditional diffusion model (baseline).
SCDM approximates the following conditional score:
\begin{equation}
    \nabla_{\xt} \log q(\xt|\yt) = \nabla_{\xt} \log q(\xt) + \nabla_{\xt} \log q(\yt|\xt),
\end{equation}
while the baseline approximates the following:
\begin{equation}
    \nabla_{\xt} \log q(\xt|\ys) = \nabla_{\xt} \log q(\xt) + \nabla_{\xt} \log q(\ys|\xt).
\end{equation}
As the unconditional score $\nabla_{\xt} \log q(\xt)$ is identical for both models, 
we analyze the relationship between class guidance (gradients of implicit classifiers~\cite{ho2022classifier,dhariwal2021diffusion}), \ie, $\nabla_{\xt}\log q(\yt|\xt)$ and $\nabla_{\xt} \log q(\ys|\xt)$ in the following proposition.
\begin{proposition} 
\label{proposition2}
Suppose there exists a differentiable function $f_i$ such that $q(\ys^{i}|\xt) = \text{Cat}(\ys^{i};p=f_i(\xt))$  and $\ys^{1}|\xt,...,\ys^{H\times W}|\xt$ are independent, where $i \in \{1,...,H\times W\}$ denotes 
the index of a pixel in semantic map and $\ys^{i} \in \mathbb{R}^{C+1}$ is a one-hot vector. 
With $q(\yt^{i}|\ys^{i})$ from Eq.~\eqref{eq:marginal} and $\gamma_{t,c} = \gamma_t  \text{ for any } c$, 
we have the following relationship;
\begin{equation*}
    \EE_{q(\yt|\ys)}[\nabla_{\xt}\log q(\yt|\xt)] =(1-\gamma_t)\nabla_{\xt}\log q(\ys|\xt).
\end{equation*}
\end{proposition}

The proof is available in Appendix~\ref{sec:supp-cfg-scale}.
Proposition \ref{proposition2} implies that the expectation of implicit classifier gradients in our method over all possible $\yt$ given $\ys$ is equivalent to $\nabla_{\xt}\log q(\ys|\xt)$ after time-dependent scaling.
First, the scaling factor $(1-\gamma_t)$ starts from 0 when $t=T$ and is set to 1 when $t=0$. 
In other words, our method acts like an unconditional generation at the beginning of the reverse process and gets stronger guidance as $t$ goes to 0.
Second, the expectation of the classifier gradient in our method is the same direction as the one in the baseline with fixed class labels. 
Note that this does \emph{not} mean that our method has the same guidance as the baseline with fixed labels and time-dependent scaling.
For more discussion, see Appendix~\ref{sec:supp-cfg-scale}.
\section{Experiments}
\label{sec:experiments}
\subsection{Noisy SIS benchmark} 
We evaluate our method based on ADE20K~\cite{zhou2017scene} dataset.
ADE20K contains 20K images for training and 2K images for test annotated with 151 classes including the `unlabeled' class.
Additionally, we introduce three new experimental setups to assess generation performance under noisy conditions using the ADE20K dataset as follows:

\textbf{[DS]} This setup employs \textit{downsampled} semantic maps that are resized by nearest-neighbor interpolation.
This setup simulates human errors such as jagged edges and coarse/low-resolution user inputs.
We downsample the semantic maps to 64$\times$64 and then upsample them to 256$\times$256.
Consequently, the label maps contain jagged edges.

\textbf{[Edge]} This setup masks the \textit{edges} of instances with an unlabeled class.
This setup imitates incomplete annotations around edges, especially between instances. 
We observed that human annotators occasionally leave the pixels on the boundaries of instances as `unlabeled' due to their inherent ambiguity, see Appendix~\ref{sec:supp-D} for examples.
Assuming a pixel with a different class compared to its neighbor is the edge of the instance, we detect edges using the label map.
Then, we fill the semantic map pixels with a distance of 2 or less from the edges with the unlabeled class.

\textbf{[Random]} This setup \textit{randomly} adds an unlabeled class to the semantic maps (10\%).
This setup mimics unintended user errors and extreme random noise.

More experimental results with other benchmark datasets (\eg, CelebAMask-HQ~\cite{lee2020maskgan} and COCO-Stuff~\cite{caesar2018coco}) are in Appendix~\ref{sec:supp-E}.
\subsection{Experimental Setup}
We adopt Fréchet Inception Distance (FID) \cite{heusel2017gans} to evaluate generation quality and mean Intersection-over-Union~(mIoU) to assess the alignment of the synthesis results with ground truth semantic maps.
We compare our model with GAN-based methods and DM-based methods.
All the baselines and ours are trained with clean benchmark datasets, and tested on the noisy SIS benchmark.
We present our results on noisy labels sampled over 25 steps in Section~\ref{sec:noisy_results}.
More information on baselines and implementation details are provided in Appendix~\ref{sec:supp-baselines} and \ref{sec:supp-implementation}, respectively.
\subsection{SIS with Noisy Labels} 
\label{sec:noisy_results}
\begin{table}[t]
\caption{
    \textbf{SIS with noisy labels.} F is FID and lower is better. M is mIoU and higher is better. The bottom three rows are DM-based, and the others are GAN-based methods. The best results among the diffusion-based approaches are boldfaced, and the best results overall are underlined.}
    \label{tab:sis-noisy}   
    \vskip 0.15in
    \centering
    \setlength{\tabcolsep}{3.5pt}
    \resizebox{\columnwidth}{!}{
    \begin{tabular}{c|l|c c| c c| c c}
        \toprule
        \multicolumn{1}{c|}{\multirow{2}{*}{\textbf{Category}}} &
        \multicolumn{1}{c|}{\multirow{2}{*}{\textbf{Methods}}}& \multicolumn{2}{c|}{\textbf{DS}} & \multicolumn{2}{c|}{\textbf{Edge}} &  \multicolumn{2}{c}{\textbf{Random}} \\
           & & \textbf{F}($\downarrow$) & \textbf{M}($\uparrow$) &  \textbf{F}($\downarrow$) & \textbf{M}($\uparrow$)&  \textbf{F}($\downarrow$) & \textbf{M}($\uparrow$) \\
        \midrule
        \midrule
        \multirow{8}{*}{GAN} &
        SPADE   &36.6 & 41.0 & 40.2  & 39.7 & 92.7 & 18.5   \\
        &CC-FPSE   &  42.0  & 40.4 & 38.7 & 41.7 & 141.3 & 16.2  \\
        &DAGAN & 37.6 & 41.3 & 42.2 & 39.0 & 109.1 & 16.4  \\
        &GroupDNet & 45.9  & 28.9 & 49.9 & 28.0 & 76.0 & 21.2  \\
        &OASIS  & 34.5  & 48.1 & {37.6} & \underline{47.4} & {54.2} & {42.1}  \\
        &CLADE & 37.4 & 41.8 & 41.9  & 40.1 & 71.8 & 29.9  \\
        &INADE  & {36.1} & 38.3 & 40.3 & 35.6 & 61.3 & 30.2  \\
        &SCGAN   &  38.1  & 44.0 & 63.9  & 43.3 & 40.9 & 38.6  \\ 
        &SAFM*   & 36.7  & \underline{50.1} & 40.0 & 44.9 & 80.6 & 34.1   \\
        \midrule
        \multirow{3}{*}{DM}&
        SDM  & 35.5 & 43.8 & 39.4 & 39.4 & 141.9 & 11.8  \\
        &LDM & 38.9 & 28.1 & 39.5 & 26.0 & 36.3 & 27.1  \\
        % \cline{2-11}
        &\cellcolor[gray]{0.9}\textbf{Ours} &  \cellcolor[gray]{0.9}\underline{\textbf{32.4}} & \cellcolor[gray]{0.9}\textbf{44.7} & \cellcolor[gray]{0.9}\underline{\textbf{31.2}} & \cellcolor[gray]{0.9}\textbf{40.1} & \cellcolor[gray]{0.9}\underline{\textbf{28.1}}  & \cellcolor[gray]{0.9}\underline{\textbf{45.2}} \\
        \bottomrule
    \end{tabular}}
\vspace{-5mm}
\end{table}
\begin{figure}[t]
\centering
\subfigure[Masks with jagged edges (DS)]{\includegraphics[width=\columnwidth]{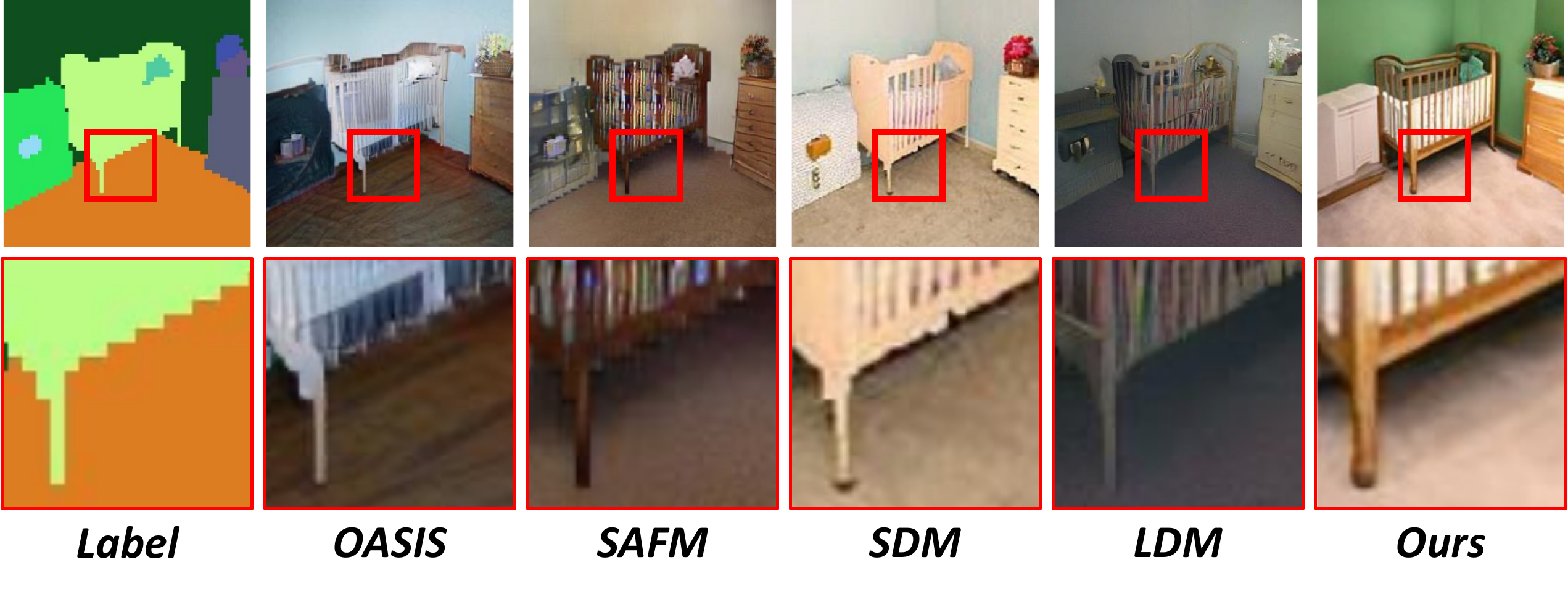}
\label{fig:noisy-low}}
\subfigure[Incomplete masks (Edge), \textcolor{green}{limegreen} areas denote `unlabeled'.]{\includegraphics[width=\columnwidth]{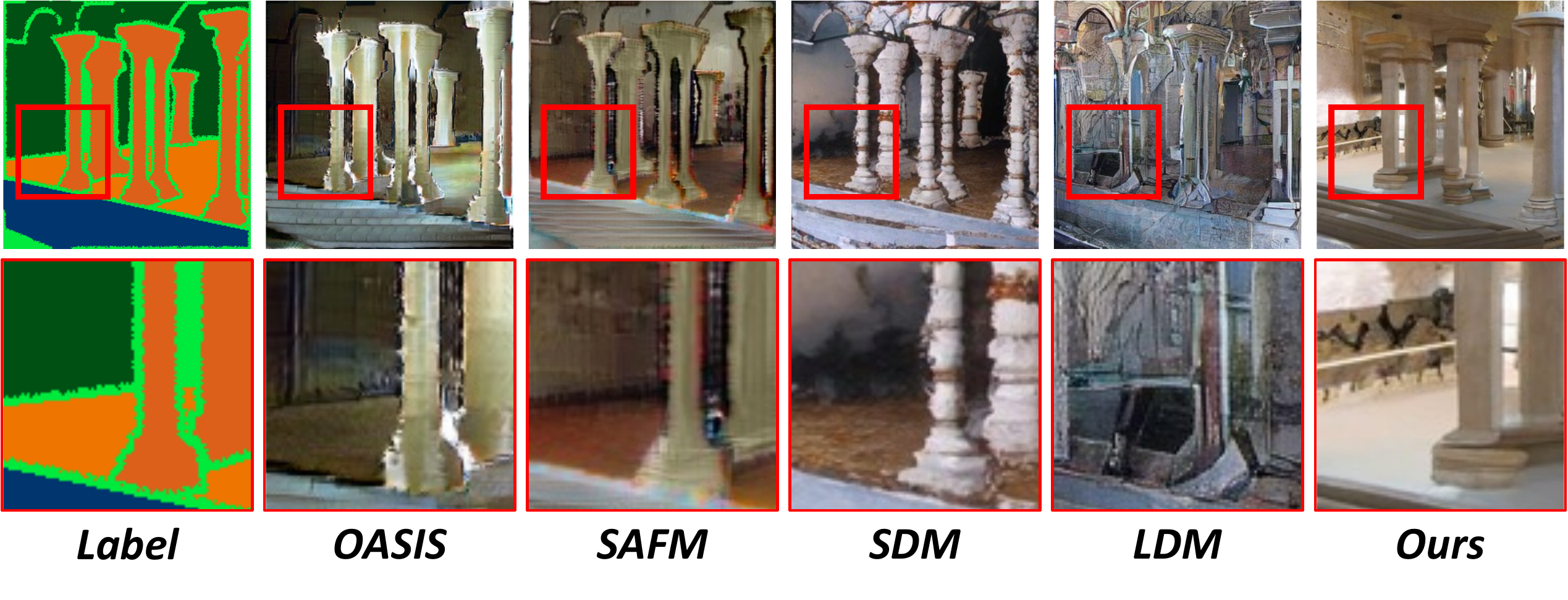}
\label{fig:noisy-edge}}
\subfigure[Corrupted masks (Random)]{\includegraphics[width=\columnwidth]{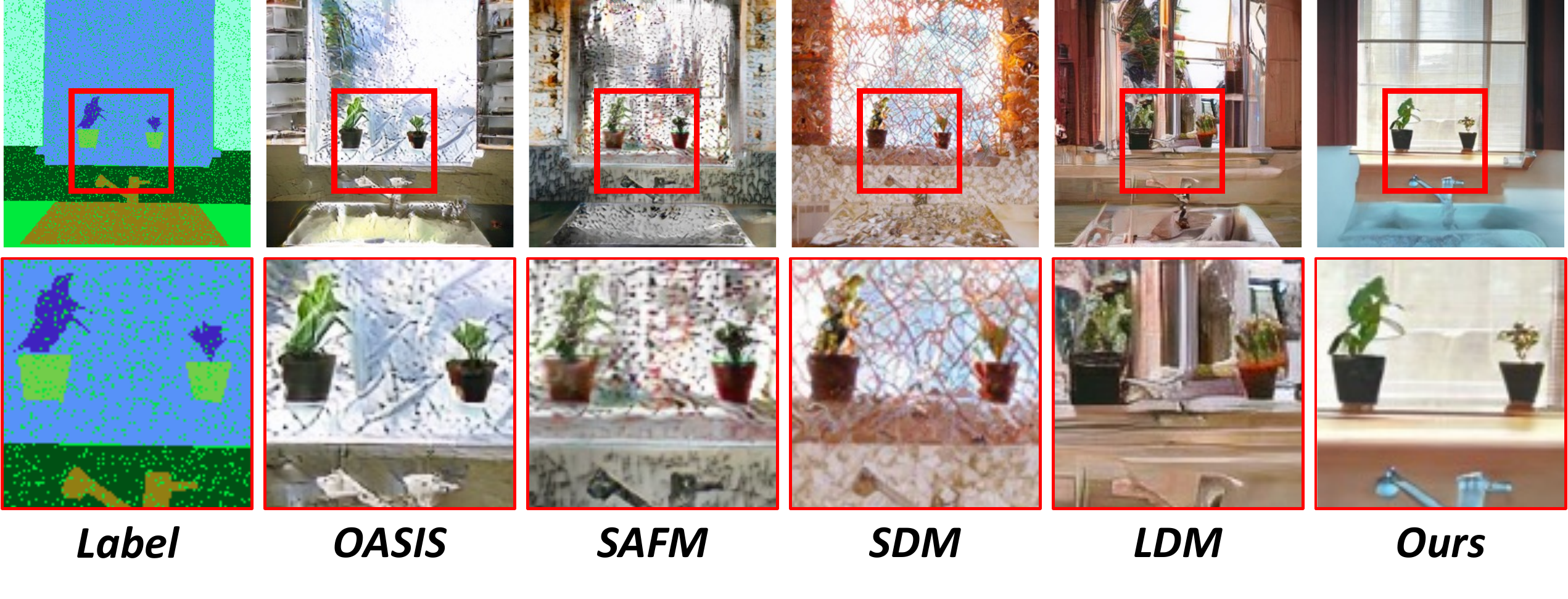}
\label{fig:noisy-rand}}
\vspace{-3mm}
\caption{\textbf{Generation results on noisy labels.}
}
\vspace{-5mm}
\label{fig:noisy-qual}
\end{figure}
Table~\ref{tab:sis-noisy} shows the performances of our method and baselines under the three settings.
SCDM demonstrates its superior robustness to all three types of noise in generation quality measured by FID, compared to other baselines.
Notably, in the three setups, the performance gaps between the best baseline scores and our method are +2.1, +6.4, and +8.2, respectively.
We also evaluate the semantic correspondence between the clean ground-truth label maps and the generation results and report mIoU scores.
Our method achieved the best mIoU performance among DM-based models in all three settings.
Compared to strong GAN-based baselines, including SAFM~\cite{lv2022semantic} 
(denoted with `*' in Table~\ref{tab:sis-noisy}) 
that leverages extra ground truth instance maps during sampling, our results show comparable results in correspondence.

We present the qualitative comparisons in Figure~\ref{fig:noisy-qual}.
While the baselines synthesized the jagged or unnatural crib images given the low-resolution semantic label, ours produced clean edges and generated a realistic image, as shown in Figure~\ref{fig:noisy-low}.
In Figure~\ref{fig:noisy-edge}, our approach naturally fills in the unlabeled edge areas, whereas the baselines fail to generate realistic images, especially on the `unlabeled' edges.
As shown in Figure~\ref{fig:noisy-rand}, ours successfully generates when conditioned on randomly corrupted masks, while others fail and synthesize artifacts.
\subsection{SIS with Clean Labels (Standard SIS)}
\label{sec:standard-sis}
\begin{table*}[t]
    \caption{\textbf{Quantitative performance comparison on generation quality.} 
    The baseline methods are categorized into Generative Adversarial Networks~(GAN) and Diffusion Models~(DM). 
    For FID, lower is better. 
    For LPIPS and mIoU, higher is better. 
    The best results among the diffusion-based approaches are boldfaced, and the best results overall are underlined.
    `-' indicates that the method did not report the metric or train the dataset, or the checkpoint or samples are not publicly available.
    `Seg' denotes that the method leverages a pre-trained segmentation network during training.
    `†' denotes a one-shot method.}
    \label{tab:sis-main}
    \vskip 0.15in
    \centering
    \setlength{\tabcolsep}{3.5pt}
    \resizebox{\textwidth}{!}{
    \begin{tabular}{c|l|c|c c c|c c c|c c c}
        \toprule
        \multicolumn{1}{c|}{\multirow{2}{*}{\textbf{}}} & \multicolumn{1}{c|}{\multirow{2}{*}{\textbf{Methods}}} & \multicolumn{1}{c|}{\multirow{2}{*}{\textbf{Seg}}} & \multicolumn{3}{c|}{\textbf{CelebAMask-HQ}} & \multicolumn{3}{c|}{\textbf{ADE20K}} &  \multicolumn{3}{c}{\textbf{COCO-Stuff}} \\
          & & & \textbf{FID}($\downarrow$) & \textbf{LPIPS}($\uparrow$)& \textbf{mIoU}($\uparrow$) &  \textbf{FID}($\downarrow$) & \textbf{LPIPS}($\uparrow$) & \textbf{mIoU}($\uparrow$)&  \textbf{FID}($\downarrow$) & \textbf{LPIPS}($\uparrow$) & \textbf{mIoU}($\uparrow$) \\
        \midrule
        \midrule
        \multirow{14}{*}{GAN} 
        & RESAIL~\cite{shi2022retrieval}    & \cmark  &  -  & -    & -    & 30.2 &  -    & 49.3* & 18.3  & - & 44.7   \\
        & SAFM~\cite{lv2022semantic}      & \cmark  &  -  & -    & -    & 32.8 &  -    & 52.6 & 24.6  & - & 43.3   \\
        & ECGAN~\cite{tang2023edge}      & \cmark  &  -  & -    & -    & 25.8 &  0.52    & 50.6* & 15.7  & - & 46.3 \\
        & ECGAN++~\cite{tang2023edge2}      & \cmark  &  -  & -    & -    & \underline{24.7} &  \underline{0.54}    & \underline{52.7}* & \underline{14.9}  & - & \underline{47.9} \\
        \cmidrule{2-12}
        & pix2pixHD~\cite{wang2018high} & \xmark & 38.5 & 0    & 76.1 & 81.8 & 0     & 20.3*  & 111.5 & 0 & 14.6   \\
        & SPADE~\cite{park2019semantic}     & \xmark &  29.2 & 0    & 75.2 & 33.9 & 0     & 44.5 & 33.9  & 0 & 36.9   \\
        & CC-FPSE~\cite{liu2019learning}   & \xmark &  -    & -    & -    & 31.7 & 0.078 & 47.3 & 19.2  & 0.098 & 40.8   \\
        & DAGAN~\cite{tang2020dual}     & \xmark & 29.1 & 0    & 76.6 & 31.9 & 0     & 45.5 & -     & - & -   \\
        & GroupDNet~\cite{zhu2020semantically} & \xmark & 25.9 & 0.365& 76.1 & 41.7 & 0.230 & 33.7 & -   & - & -   \\
        & OASIS~\cite{sushko2020you}     & \xmark &   -  & -    & -    & 28.3 & 0.286 & 50.9 & 17.0  & 0.328 & 44.2   \\ 
        & INADE~\cite{tan2021diverse}     &\xmark & 21.5 & 0.415& 74.1 & 35.2 & 0.459 & 41.4 & -   & - & -   \\
        & SCGAN~\cite{wang2021image}     & \xmark & 20.8 & 0    & 75.5 & 29.3 & 0     & 50.0 & 18.1  & 0 & 41.7   \\
        & CLADE~\cite{tan2021efficient}     & \xmark & 30.6 & 0    & 75.4 & 35.4 & 0     & 44.7 & 29.2  & 0 & 36.9   \\
        & Unconditional Generator†~\cite{chae2024semantic}      & \xmark  &  18.5  & -    & 53.1    & - &  -    & - & -  & - & - \\
        \midrule
        \multirow{4}{*}{DM}& SDM~\cite{wang2022semantic}       & \xmark & 18.8 & {0.404} & 77.0 & 27.5 &0.524 & 48.7 & 15.9  & 0.518 & 34.9   \\
        & LDM~\cite{rombach2022high}  & \xmark & 21.5 & 0.315 & 74.6 & 36.5 & 0.417 & 23.2 & - & - & - \\
        & PITI~\cite{wang2022pretraining}  & \xmark &  - & - & - & 27.3 & - & - & 15.8 & 0.489 & 32.2 \\
         & \cellcolor[gray]{0.9}\textbf{Ours} &  \cellcolor[gray]{0.9}{\xmark} 
         & \cellcolor[gray]{0.9}{\underline{\textbf{17.4}}} & \cellcolor[gray]{0.9}\underline{\textbf{0.418}} & \cellcolor[gray]{0.9}\underline{\textbf{77.2}} & \cellcolor[gray]{0.9}{\textbf{26.9}} & \cellcolor[gray]{0.9}{\textbf{0.530}} & \cellcolor[gray]{0.9}\textbf{49.4} & \cellcolor[gray]{0.9}\textbf{15.3} & \cellcolor[gray]{0.9}{\underline{\textbf{0.519}}} & \cellcolor[gray]{0.9}\textbf{38.1} \\
        \bottomrule
    \end{tabular}}
\end{table*}
We also evaluate our method in a standard SIS setting with clean labels.
In this experiment, we additionally adopt LPIPS~\cite{zhang2018unreasonable} as a diversity metric.
Table~\ref{tab:sis-main} shows that SCDM achieves the best performance in all three metrics (\textit{e.g.}, FID, LPIPS, mIoU) compared to recent DM-based baselines in all datasets.
Also, including GAN-based models, the proposed method shows comparable performances.
Specifically, our method surpasses all baselines on CelebAMask-HQ in all three metrics with a significant gain of +1.1~(FID) compared to the state-of-the-art method.
Qualitative comparisons and more detailed analysis are in Appendix~\ref{sec:exp_results}.
\section{Analysis}
In this section, we analyze our method to understand (1) the efficacy of Label diffusion and extrapolation
and (2) the effect of the class-wise noise schedule.
\subsection{Ablation Study} 
\label{sec:ablation}
\begin{table*}[t]
    \caption{\textbf{Ablation study on \textit{ADE20K}.} For FID~(F), lower is better. For LPIPS~(L) and mIoU~(M), higher is better.}
    \label{tab:ablation}
    \begin{center}
    \setlength{\tabcolsep}{3pt}
    \resizebox{\textwidth}{!}{
    \begin{tabular}{c|l| c| c | c | c | c | c | c | c | c | c | c | c | c | c | c}
        \toprule
        & \multirow{2}{*}{\textbf{Method}} & \multicolumn{3}{c|}{\textbf{25 steps}} & \multicolumn{3}{c|}{\textbf{50 steps}} & \multicolumn{3}{c|}{\textbf{100 steps}} & \multicolumn{3}{c|}{\textbf{250 steps}} & \multicolumn{3}{c}{\textbf{1000 steps}}\\
         & 
& \textbf{F}($\downarrow$) & \textbf{L}($\uparrow$) & \textbf{M}($\uparrow$)& \textbf{F}($\downarrow$) & \textbf{L}($\uparrow$) & \textbf{M}($\uparrow$)& \textbf{F}($\downarrow$) & \textbf{L}($\uparrow$) & \textbf{M}($\uparrow$)& \textbf{F}($\downarrow$) & \textbf{L}($\uparrow$) & \textbf{M}($\uparrow$)& \textbf{F}($\downarrow$) & \textbf{L}($\uparrow$) & \textbf{M}($\uparrow$)\\
        \midrule
        \midrule
        (a) & Base 
        & 44.6 & 0.471 & 33.8 & 35.8 & 0.489 & 47.1 & 31.9 & 0.500 & 48.2 & 29.3 & 0.506 & 48.6 & 28.1 & 0.508 & 48.6 \\
        (b) & + Label Diffusion
        & 39.5 & 0.492 & 45.5 & 33.6 & 0.513 & 47.3 & 29.8 & \textbf{0.522} & 48.6 & 27.7 & 0.528 & 48.7 & 26.9 & 0.530 & 48.8 \\
        (c) & + Extrapolation
        & \textbf{27.7} & \textbf{0.518} & \textbf{48.7} & \textbf{27.0} & \textbf{0.525} & \textbf{48.7} & \textbf{26.7} & \textbf{0.522} & \textbf{49.8} & \textbf{26.7} & \textbf{0.530} & \textbf{49.6} & \textbf{26.8} & \textbf{0.531} & \textbf{49.9} \\
        \bottomrule
    \end{tabular}
    }
    \end{center}
\end{table*}
\begin{table}[t]
\caption{\textbf{Ablation study on noisy SIS.} Generation results with and without Label Diffusion are compared, where samples generated from each noisy dataset are compared with those from the clean dataset.
For LPIPS and FID, lower is better. For SSIM and PSNR, higher is better.}
    \label{tab:noisy-ablation} 
    \begin{center}
    \setlength{\tabcolsep}{3pt}
    \resizebox{\columnwidth}{!}{
    \begin{tabular}{c|c|c|c|c|c}
        \toprule
        \multirow{2}{*}{\textbf{Dataset}} & \multirow{2}{*}{\textbf{Method}} & \multicolumn{4}{c}{\textbf{Metric}} \\
        & & LPIPS($\downarrow$) & SSIM($\uparrow$) & PSNR($\uparrow$) & FID($\downarrow$) \\
        \midrule
        \midrule
        \multirow{2}{*}{\textbf{DS}} & Baseline & 0.221 & 0.823 & 30.4 & 24.2\\
        &\cellcolor[gray]{0.9}Ours & \cellcolor[gray]{0.9}\textbf{0.180} & \cellcolor[gray]{0.9}\textbf{0.865} & \cellcolor[gray]{0.9}\textbf{31.3} & \cellcolor[gray]{0.9}\textbf{19.0}\\
        \midrule
        \multirow{2}{*}{\textbf{Edge}} & Baseline & 0.248 & 0.771 & 29.7 & 32.7 \\
        &\cellcolor[gray]{0.9}Ours& \cellcolor[gray]{0.9}\textbf{0.223}& \cellcolor[gray]{0.9}\textbf{0.825}& \cellcolor[gray]{0.9}\textbf{30.2}& \cellcolor[gray]{0.9}\textbf{20.0}\\
        \midrule
        \multirow{2}{*}{\textbf{Random}} & Baseline & 0.560 & 0.427 & 28.1 & 145.6 \\
        &\cellcolor[gray]{0.9}Ours& \cellcolor[gray]{0.9}\textbf{0.076}& \cellcolor[gray]{0.9}\textbf{0.944}& \cellcolor[gray]{0.9}\textbf{32.9}& \cellcolor[gray]{0.9}\textbf{10.0}\\
        \bottomrule
    \end{tabular}
    }
    \end{center}
\end{table}
\begin{figure}[t!]
\centering
\includegraphics[width=1.0\columnwidth]{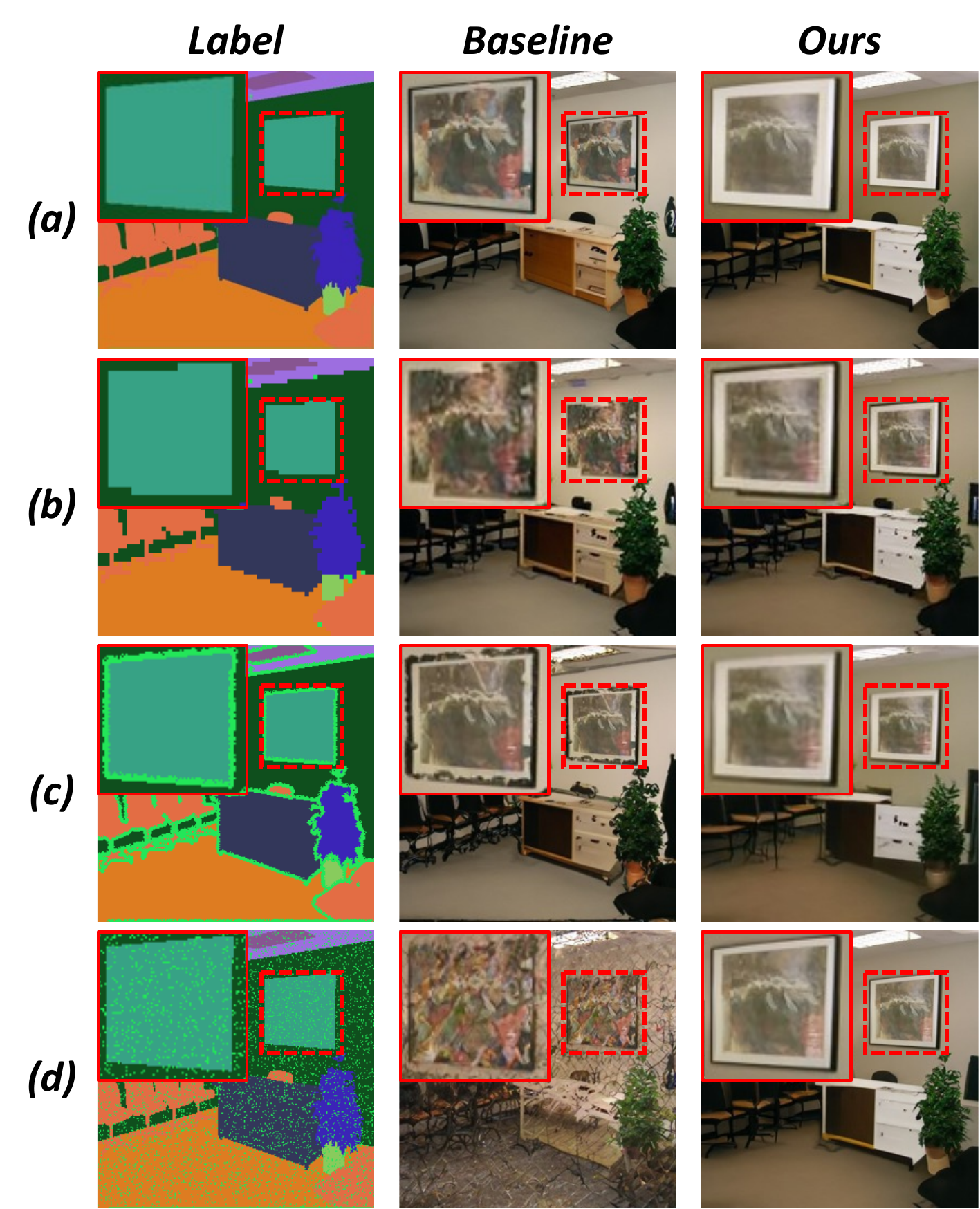}
\caption{\textbf{Generation results with and without Label Diffusion.} The results are sampled with the fixed random seeds and the same $\xb_T$, and generated with (a) clean labels, (b) DS, (c) Edge, and (d) Random setup noisy labels, respectively.
}
\label{fig:noisy-ablation}
\end{figure}
To further demonstrate SCDM's enhanced robustness against noisy labels, we first show the effect of our components in the original (clean) benchmark.
Subsequently, we show that our samples generated with noisy labels closely resemble the results obtained with clean labels, supporting our motivation in Figure~\ref{fig:motivation2}.

\subsubsection{Does SCDM successfully estimate $q(X|Y)$?}
\label{sec:abl1}
The effect of two components of SCDM (Label Diffusion and extrapolation) is analyzed by an ablation study in Table~\ref{tab:ablation}.
In this analysis, we adopt LPIPS~\cite{zhang2018unreasonable} to additionally compare the generation diversity and measure the average distance between multi-modal synthesis results.
We compare the following;
(a) Base generates images conditioned on original fixed semantic maps $\ys$, 
(b) +Label Diffusion generates images conditioned on perturbed labels by our Label Diffusion, and 
(c) +Extrapolation uses Eq.~(\ref{eq:extrapolation}) instead of direct sampling of $\mathbf{x}_{t-1}$ on top of (b).
The results are reported for the few-step (25, 50, 100, and 250 steps) and the full-step (1000 steps) settings.

By eliminating all of our components, (c) $\rightarrow$ (a), the performance of all three metrics significantly declined for all sampling steps.
This empirically shows that our SCDM successfully estimates the conditional distribution $q(X|Y)$.
The degradation became more substantial when omitting extrapolation, (b) $\rightarrow$ (a), in fewer steps, highlighting its significance in the few-step generation. 
Surprisingly, full SCDM (c) with only 25 steps (FID of 27.7) outperforms the baseline (a) with 1000 steps (FID of 28.1).
This supports the effectiveness of the proposed method.

\subsubsection{Does Label Diffusion contribute to robustness?}
\label{sec:abl2}
Furthermore, we examine the effect of Label Diffusion in the noisy SIS setting to verify the robustness of our method, as depicted in Figure~\ref{fig:motivation2}.
Specifically, we demonstrate that the generation results given clean and noisy labels are similar.
To rigorously analyze the effects, we fix the random seeds and use the same $\xb_T$.
Then we compare the generation results of our method and the baseline, conditioned on clean (original dataset) semantic maps with results of noisy (DS, Edge, and Random) semantic maps, respectively.
For a quantitative comparison, we adopt the following metrics: LPIPS (perceptual similarity), SSIM (structural similarity), PSNR (peak signal-to-noise ratio), and FID, and details are in Appendix~\ref{sec:supp-metrics}.
LPIPS, SSIM, and PSNR are calculated \textit{sample}-wise, comparing each pair of generated samples of the clean and noisy semantic maps, while FID compares the distribution of the generated \textit{set} of images.

Results in Table~\ref{tab:noisy-ablation} indicate that Label Diffusion significantly contributes to robust generation, with our samples exhibiting better similarity in all four metrics.
Particularly, ours resulted in +0.041, +0.025, and +0.484 of LPIPS gain over baseline in DS, Edge, and Random settings, respectively.
Figure~\ref{fig:noisy-ablation} presents a qualitative comparison supporting our intuition behind SCDM.
Without Label Diffusion (Baseline in Figure~\ref{fig:noisy-ablation}), the generated results show inconsistency, whereas samples conform to the result of clean labels when Label Diffusion is employed (Ours in Figure~\ref{fig:noisy-ablation}).
\subsection{Effect of Class-wise Noise Schedule}
\label{sec:noise_schedule_analysis}
\begin{table}[t]
    \caption{\textbf{mIoU per each group on \textit{ADE20K}.}
    The classes are grouped based on their $\psi_c\phi_c$ scores.}
    \label{tab:group-iou}
    \begin{center}
    \setlength{\tabcolsep}{3pt}
    \resizebox{\columnwidth}{!}{
    \begin{tabular}{l| c | c | c | c}
        \toprule
        \textbf{Noise Schedule} & \textbf{All} & \textbf{Frequent} & \textbf{Common} & \textbf{Rare} \\
        \midrule
        \midrule
        \textbf{Linear \& uniform} & 43.0 & 56.1 & 40.8 & 32.2 \\
        \midrule
        \textbf{Class-wise} & \textbf{49.4} (+6.4) & \textbf{60.3} (+4.2) & \textbf{47.7} (+6.9) & \textbf{38.4} (+8.1) \\
        \bottomrule
    \end{tabular}
    }
    \end{center}
    \vspace{-5mm}
\end{table}
In this section, we elucidate the effect of our class-wise noise schedule by comparing the generation results of two different models using the class-wise schedule and linear and uniform schedule, \ie, $\eta \rightarrow 0$.
We observe that the class-wise schedule clearly improves the image quality, especially in terms of semantic correspondence (mIoU~($\uparrow$) of \textbf{49.4} (class-wise) $>$ 43.0 (uniform) on ADE20K).
The quantitative and qualitative results are presented in the Appendix~\ref{sec:supp-classwise}.
Furthermore, the class-wise schedule exhibits superiority in small and rare class synthesis.
By organizing the classes into three groups - frequent, common, and rare - based on their $\psi_c \phi_c$ scores, we compare mIoU per each group on ADE20K and report the performance in Table~\ref{tab:group-iou}.
Notably, our class-wise schedule exhibited mIoU gain in all the groups, with the highest performance improvement of +8.1 in the rare group.
\section{Conclusion}
This paper introduces SCDM, a novel and robust conditional diffusion model for semantic image synthesis.
The discrete diffusion for labels, which we name Label Diffusion, is designed with label masking.
Label Diffusion ensures that the intermediate labels along the generation process become similar and eventually identical at $t=T$.
Additionally, the class-wise noise schedule improves the generation quality for small and rare objects.
We define the generation process with a discrete forward process of labels and a continuous reverse process of images, as the labels are given in SIS.
SCDM demonstrates its robustness in noisy SIS setups which we designed to simulate human errors in real-world applications.
\section*{Acknowledgements}{
This work was partly supported by ICT Creative Consilience Program through the Institute of Information \& Communications Technology Planning \& Evaluation (IITP) grant funded by the Korea government (MSIT)(IITP-2024-2020-0-01819, 10\%), the National Research Foundation of Korea (NRF) grant funded by the Korea government (MSIT)(NRF-2023R1A2C2005373, 30\%), the Virtual Engineering Platform Project (Grant No. P0022336, 30\%), funded by the Ministry of Trade, Industry \& Energy (MoTIE, South Korea), and the National Supercomputing Center with supercomputing resources including technical support (KSC-2022-CRE-0261, 30\%).
}
\section*{Impact Statement}{
This paper focuses on robust generation against noisy user input, a crucial aspect in real-world applications.
Since clean labels and benchmark datasets inherently contain biases, our model might reinforce these biases when its samples are widely used across the Internet.
% As the clean labels and the benchmark dataset intrinsically contain biases, our model might reinforce the biases when the samples from our model are used prevalently throughout the Internet.
Nevertheless, our research opens new avenues for future work in image generation for real-world problems, potentially expanding the application of diffusion models in practical scenarios.
}
\bibliography{main}
\bibliographystyle{icml2024}

%%%%%%%%%%%%%%%%%%%%%%%%%%%%%%%%%%%%%%%%%%%%%%%%%%%%%%%%%%%%%%%%%%%%%%%%%%%%%%%
%%%%%%%%%%%%%%%%%%%%%%%%%%%%%%%%%%%%%%%%%%%%%%%%%%%%%%%%%%%%%%%%%%%%%%%%%%%%%%%
% APPENDIX
%%%%%%%%%%%%%%%%%%%%%%%%%%%%%%%%%%%%%%%%%%%%%%%%%%%%%%%%%%%%%%%%%%%%%%%%%%%%%%%
%%%%%%%%%%%%%%%%%%%%%%%%%%%%%%%%%%%%%%%%%%%%%%%%%%%%%%%%%%%%%%%%%%%%%%%%%%%%%%%
\newpage
\appendix
\onecolumn
The appendix is organized into the following sections.
\begin{itemize}
    \item Appendix~\ref{sec:supp-A}: Derivations
        \begin{itemize}
            \item \ref{sec:supp-vlb-proof} Variational Lower Bound for Stochastic Conditional Diffusion Model
            \item \ref{sec:supp-lemma1-proof} Proof of Proposition~\ref{lemma1}
            \item \ref{sec:supp-prop2-proof} Proof of Proposition~\ref{proposition2}
        \end{itemize}
    \item Appendix~\ref{sec:supp-B}: Algorithms
    \item Appendix~\ref{sec:supp-C}: Experimental Setup
        \begin{itemize}
            \item \ref{sec:supp-metrics} Metrics
            \item \ref{sec:supp-baselines} Baselines
            \item \ref{sec:supp-implementation} Implementation Details
            \item \ref{sec:supp-trajectory} Efficient Trajectory Representation
        \end{itemize}
    \item Appendix~\ref{sec:supp-psiphi}: Detailed Explanations on $\psi_c$ and $\phi_c$
    \item Appendix~\ref{sec:supp-D}: ADE20K Dataset Annotation Examples
    \item Appendix~\ref{sec:limitations}: Limitations
    \item Appendix~\ref{sec:supp-E}: Additional Experimental Results
        \begin{itemize}
            \item \ref{sec:exp_results} Standard SIS Setting
            \item \ref{sec:supp-extr-w} Extrapolation Hyperparameter Search
            \item \ref{sec:supp-vis} Visualization of Label Diffusion
            \item \ref{sec:supp-classwise} Effect of Class-wise Noise Schedule
            \item \ref{sec:supp-cfg-scale} Further Discussion and Analysis on Class Guidance
            \item \ref{sec:supp-validation-gen} Validation of SCDM Generation Process
            \item \ref{sec:supp-multimodal} More Qualitative Results - Multimodal Generation
            \item \ref{sec:supp-more-noisy} More Qualitative Results - SIS with Noisy Labels
            \item \ref{sec:supp-more-standard} More Qualitative Results - Standard SIS Setting
        \end{itemize}
\end{itemize}
\section{Derivations}
\label{sec:supp-A}
\subsection{Variational Lower Bound for Stochastic Conditional Diffusion Model}
\label{sec:supp-vlb-proof}
In this section, we provide a detailed derivation of the objective function~(Eq. \eqref{eq:loss}) discussed in Section 4.3 of the main paper.
We start by defining our Stochastic Conditional Diffusion Model as:
\begin{equation}
    \ptheta(\xs|\ys) := \int\int\ptheta(\xsT,\yfT|\ys)\mathrm{d\xb}_{1:T}\mathrm{d\yb}_{1:T},
\end{equation}
where the generation process is defined as follows:
\begin{align}
    \ptheta(\xsT,\yfT|\ys) &:= p(\xT)q(\yfT|\ys)\prod^T_{t=1}\ptheta(\xtt|\xt, \yt), \\
    \ptheta(\xtt|\xt,\yt) &:= \N(\xtt; \mathbf{\mu}_\theta(\xt,\yt,t), \mathbf{\Sigma}_\theta(\xt,\yt,t)).
\end{align}
The diffusion forward process in which SCDM approximates is:
\begin{align}
    q(\xfT,\yfT|\xs,\ys) &:= \prod^T_{t=1}q(\xt,\yt|\xtt, \ytt), \\
    q(\xt,\yt|\xtt,\ytt) &:= q(\xt|\xtt)q(\yt|\ytt).
\end{align}
Variational bound on negative log-likelihood of SCDM can be derived as follows:

\begin{align}
    & \EE_{q(\xfT,\yfT|\xs,\ys)} \ll-\log p_\theta(\xs|\ys) \rr  \\
    &= \EE_q \ll-\log \frac{p_\theta(\xsT,\yfT|\ys)}{p_\theta(\xfT,\yfT|\ys)} \rr \\
    & = \EE_q \ll-\log \frac{p_\theta(\xsT,\yfT|\ys)}{p_\theta(\xfT,\yfT|\ys)}\frac{q(\xfT,\yfT|\xs,\ys)}{q(\xfT,\yfT|\xs,\ys)} \rr \\
    & = \EE_q \ll-\log \frac{p_\theta(\xsT,\yfT|\ys)}{q(\xfT,\yfT|\xs,\ys)}\frac{q(\xfT,\yfT|\xs,\ys)}{p_\theta(\xfT,\yfT|\ys)} \rr \\
    &= \EE_q \ll-\log \frac{p_\theta(\xsT,\yfT|\ys)}{q(\xfT,\yfT|\xs,\ys)}\rr + \EE_q \ll-\log \frac{q(\xfT,\yfT|\xs,\ys)}{p_\theta(\xfT,\yfT|\ys)} \rr \\
    &= \EE_q \ll-\log \frac{p_\theta(\xsT,\yfT|\ys)}{q(\xfT,\yfT|\xs,\ys)}\rr - \underbrace{D_\text{KL}\left( {q(\xfT,\yfT|\xs,\ys)}||{p_\theta(\xfT,\yfT|\ys)} \right)}_{D_\text{KL}\ge 0} \\
    &\le \EE_q \ll-\log \frac{p_\theta(\xsT,\yfT|\ys)}{q(\xfT,\yfT|\xs,\ys)}\rr,
\end{align}

\begin{align}
    &\EE_q \ll-\log \frac{p_\theta(\xsT,\yfT|\ys)}{q(\xfT,\yfT|\xs,\ys)}\rr  \\
    &= \EE_q \ll-\log \frac{p(\xT)q(\yfT|\ys)\prod^T_{t=1}\ptheta(\xtt|\xt, \yt)}{\prod^T_{t=1}q(\xt,\yt|\xtt, \ytt)}\rr \\
    & = \EE_q \ll-\log \frac{p(\xT)q(\yfT|\ys)\prod^T_{t=1}\ptheta(\xtt|\xt, \yt)}{\prod^T_{t=1}q(\xt|\xtt)q(\yt|\ytt)}\rr \\
    & = \EE_q \ll-\log \frac{p(\xT)q(\yfT|\ys)\prod^T_{t=1}\ptheta(\xtt|\xt, \yt)}{\prod^T_{t=1}q(\xt|\xtt)\prod^T_{t=1}q(\yt|\ytt)}\rr \\
    & = \EE_q \ll-\log \frac{p(\xT)\prod^T_{t=1}\ptheta(\xtt|\xt, \yt)}{\prod^T_{t=1}q(\xt|\xtt)}\rr \\
    & = \EE_q \ll-\log p(\xT) -\sum^T_{t=1}\log \frac{\ptheta(\xtt|\xt, \yt)}{q(\xt|\xtt)}\rr =: L
\end{align}

Also, we can derive the reduced variance variational bound for SCDM as follows:
\begin{align}
    L & = \EE_q \ll-\log p(\xT) -\sum^T_{t=1}\log \frac{\ptheta(\xtt|\xt, \yt)}{q(\xt|\xtt)}\rr \\
    & = \EE_q \left[-\log p(\xT) -\sum^T_{t=2}\log \frac{\ptheta(\xtt|\xt, \yt)}{q(\xt|\xtt)} -\log\frac{\ptheta(\xs|\xf, \yf)}{q(\xf|\xs)} \right]\\
    & = \EE_q \left[-\log p(\xT) -\sum^T_{t=2}\log \frac{\ptheta(\xtt|\xt, \yt)}{q(\xtt|\xt,\xs)}\frac{q(\xtt|\xs)}{q(\xt|\xs)} -\log\frac{\ptheta(\xs|\xf, \yf)}{q(\xf|\xs)} \right] \\
    & = \EE_q \left[ -\log\frac{p(\xT)} {q(\xT|\xs)} -\sum^T_{t=2}\log \frac{\ptheta(\xtt|\xt, \yt)}{q(\xtt|\xt,\xs)} -\log \ptheta(\xs|\xf, \yf)\right] \\
    & = \EE_q \left[ D_\text{KL}({q(\xT|\xs)}||{p(\xT)}) + \sum^T_{t=2}D_\text{KL}({q(\xtt|\xt,\xs)}||{\ptheta(\xtt|\xt, \yt)}) -\log \ptheta(\xs|\xf, \yf) \right].
\end{align}

The only difference between the typical bound of a conditional DM and our bound is the substitution of $\ys$ for $\yt$ in $\ptheta(\xtt|\xt, \yt)$.
Therefore, we can use the DDPM simple loss~ ($\mathcal{L}_\text{simple}$)~\cite{ho2020denoising} or the hybrid loss~($\mathcal{L}_\text{simple} + \lambda \mathcal{L}_
\text{vlb}$)~\cite{nichol2021improved} for training our Stochastic Conditional Diffusion Model. 
\subsection{Proof of Proposition \ref{lemma1}}
\label{sec:supp-lemma1-proof}
\setcounter{proposition}{0}
We provide the proof of Proposition~\ref{lemma1}, which shows that our class-wise noise schedule generalizes the linear and uniform schedule and no Label Diffusion, \ie, a typical conditional DM described in Section~\ref{sec:background}.
Proposition~\ref{lemma1} and the following proof hold under two assumptions of $\psi_c \phi_c > 1$ for all $c$ and $t < T$.
\begin{proposition} 
For $\gamma_{t,c}$ in Eq.~\eqref{eq:gamma_t} with $\psi_c\phi_c > 1$ for all $c$ and $t<T$,
\begin{equation*}
    \lim_{\eta \rightarrow 0} \gamma_{t,c} = \frac{t}{T} \;\;\;\text{ and}\;\;\lim_{\eta \rightarrow \infty} \gamma_{t,c} = 0.
\end{equation*}
\end{proposition}
\begin{proof}
First, when $\eta$ converges to $0$, the linear and uniform noise schedule can be derived as follows:
\begin{align}
     \lim_{\eta \to 0}\gamma_{t,c}
     & = \lim_{\eta \to 0} \frac{(\psi_c  \phi_c)^{\eta\frac{t}{T}}-1}{(\psi_c  \phi_c)^\eta -1} \\
     & = \lim_{\eta \to 0} \frac{(\psi_c  \phi_c)^{\eta\frac{t}{T}}-(\psi_c  \phi_c)^{0\frac{t}{T}}}{(\psi_c  \phi_c)^\eta -(\psi_c  \phi_c)^0}\\
     & = \lim_{\eta \to 0} \frac{\frac{(\psi_c  \phi_c)^{\eta\frac{t}{T}}-(\psi_c  \phi_c)^{0\frac{t}{T}}}{\eta}}{\frac{(\psi_c  \phi_c)^\eta -(\psi_c  \phi_c)^0}{\eta}} \\
     & = \lim_{\eta \to 0} \frac{\ln(\psi_c  \phi_c)\frac{t}{T}(\psi_c  \phi_c)^{\eta\frac{t}{T}}}{\ln(\psi_c  \phi_c)(\psi_c  \phi_c)^{\eta}} \\
     & = \lim_{\eta \to 0} \frac{t}{T} (\psi_c  \phi_c)^{\eta(\frac{t}{T}-1)} \\
     & = \frac{t}{T}.
\end{align}

When $\eta$ explodes to $+\infty$, the noise schedule without Label Diffusion can be derived as follows:
\begin{align}
    \lim_{\eta \to +\infty}\gamma_{t,c}
     & = \lim_{\eta \to +\infty} \frac{(\psi_c  \phi_c)^{\eta\frac{t}{T}}-1}{(\psi_c  \phi_c)^{\eta} -1} \\
    & = \lim_{\eta \to +\infty} \frac{(\psi_c  \phi_c)^{\eta(\frac{t}{T}-1)}-(\psi_c  \phi_c)^{-\eta}}{1 - (\psi_c  \phi_c)^{-\eta}} \\
    & = 0. \qquad \because \frac{t}{T}-1 < 0
    % & = \begin{cases} 1 & \text { if $t=T$, }  \\ 0 & \text{ otherwise.}\end{cases} 
\end{align}
\end{proof}
These two derivations show that our class-wise noise schedule generalizes previous works. 
Although we set $\eta = 1$ for our experiments, different $\eta$ can be searched and employed for controlling the noise schedules, which we leave as future work.
\subsection{Proof of Proposition~\ref{proposition2}}
\label{sec:supp-prop2-proof}
We provide the proof of Proposition~\ref{proposition2}.
\begin{proposition} 
Suppose there exists a differentiable function $f_i$ such that $q(\ys^{i}|\xt) = \text{Cat}(\ys^{i};p=f_i(\xt))$  and $\ys^{1}|\xt,...,\ys^{H\times W}|\xt$ are independent, where $i \in \{1,...,H\times W\}$ denotes 
the index of a pixel in semantic map and $\ys^{i} \in \mathbb{R}^{C+1}$ is a one-hot vector. 
With $q(\yt^{i}|\ys^{i})$ from Eq.~\eqref{eq:marginal} and $\gamma_{t,c} = \gamma_t  \text{ for any } c$, 
we have the following relationship;
\begin{equation*}
    \EE_{q(\yt|\ys)}[\nabla_{\xt}\log q(\yt|\xt)] =(1-\gamma_t)\nabla_{\xt}\log q(\ys|\xt).
\end{equation*}
\end{proposition}
\begin{proof}
\begin{align}
q(\ys|\xt) &= q(\ys^{1},\ys^{2},...,\ys^{H\times W}|\xt) \\
&= \prod_{i=1}^{H\times W} q(\ys^{i}|\xt), \\
\nabla_{\xt}\log q(\ys|\xt) &= \nabla_{\xt}\log \prod_{i=1}^{H\times W} q(\ys^{i}|\xt) \\
&= \sum_{i=1}^{H\times W} \nabla_{\xt}\log q(\ys^{i}|\xt).
\end{align}
Since $\ys^{i}|\xt$ are independent and our Label Diffusion forward process is applied to each pixel of $\ys$ independently, $\yt^{i}|\xt$ are independent. Therefore,
\begin{align}
q(\yt|\xt) &= q(\yt^{1},\yt^{2},...,\yt^{H\times W}|\xt) \\
&= \prod_{i=1}^{H\times W} q(\yt^{i}|\xt), \\
\nabla_{\xt}\log q(\yt|\xt) &= \nabla_{\xt}\log \prod_{i=1}^{H\times W} q(\yt^{i}|\xt)\\
&= \sum_{i=1}^{H\times W} \nabla_{\xt}\log q(\yt^{i}|\xt),\\
\EE_{q(\yt|\ys)}[\nabla_{\xt}\log q(\yt|\xt)] &= \sum_{i=1}^{H\times W} \EE_{q(\yt|\ys)}[\nabla_{\xt}\log q(\yt^{i}|\xt)]\\
&= \sum_{i=1}^{H\times W} \EE_{q(\yt^{i}|\ys)}[\nabla_{\xt}\log q(\yt^{i}|\xt)]\\
&= \sum_{i=1}^{H\times W} \EE_{q(\yt^{i}|\ys^{i})}[\nabla_{\xt}\log q(\yt^{i}|\xt)].
\end{align}
Thus, our proof can be substituted for proving the following statement:
\begin{align}
    \EE_{q(\yt^{i}|\ys^{i})}[\nabla_{\xt}\log q(\yt^{i}|\xt)] =(1-\gamma_t)\nabla_{\xt}\log q(\ys^{i}|\xt).
\end{align}
With a slight abuse of notation, we use $\ys$ to denote $\ys^{i}$, $\yt$ to denote $\yt^{i}$, and $f(\xt)$ to denote $f_i(\xt)$ for the rest of the proof.
Then, $\nabla_{\xt}\log q(\ys|\xt)$ can be derived as follows:
\begin{align}
q(\ys|\xt) &= \ys^Tf(\xt), \\
\nabla_{\xt}\log q(\ys|\xt) &= \frac {1} {q(\ys|\xt)} \frac {\partial f(\xt)} {\partial \xt}\ys \\
&= \frac {1} {\ys^Tf(\xt)} \frac {\partial f(\xt)} {\partial \xt}\ys.
\end{align}
Also, $\nabla_{\xt} \log q(\yt|\xt)$ can be derived as follows:
\begin{align}
q(\yt|\xt) &= \sum_{\ys}q(\yt,\ys|\xt) = \sum_{\ys}q(\yt|\ys,\xt)q(\ys|\xt) \\
&=\sum_{\ys}q(\yt|\ys)q(\ys|\xt) =\sum_{c=1}^{C+1}\yt^T \overline{\Q}_t \mathbf{e}_c\mathbf{e}_c^Tf(\xt) \\
&= \yt^T \overline{\Q}_t (\sum_{c=1}^{C+1}\mathbf{e}_c\mathbf{e}_c^T)f(\xt)\\
&=\yt^T \overline{\Q}_t f(\xt),\\
\nabla_{\xt}\log q(\yt|\xt) &= \frac {1} {q(\yt|\xt)} \frac {\partial f(\xt)} {\partial \xt} \overline{\Q}_t^T\yt \\
&= \frac {1} {\yt^T \overline{\Q}_t f(\xt)} \frac {\partial f(\xt)} {\partial \xt} \overline{\Q}_t^T\yt.
\end{align}
Therefore, $\EE_{q(\yt|\ys)}[\nabla_{\xt}\log q(\yt|\xt)]$ can be derived as follows:
\begin{align}
\EE_{q(\yt|\ys)}[\nabla_{\xt}\log q(\yt|\xt)] 
&= \sum_{\yt} q(\yt|\ys)\frac {1} {\yt^T \overline{\Q}_t f(\xt)} \frac {\partial f(\xt)} {\partial \xt} \overline{\Q}_t^T\yt\\
&= \frac {\partial f(\xt)} {\partial \xt}~\sum_{\yt} q(\yt|\ys)\frac {1} {\yt^T \overline{\Q}_t f(\xt)}  \overline{\Q}_t^T\yt\\
&= \frac {\partial f(\xt)} {\partial \xt}~\sum_{\yt} {\yt^T \overline{\Q}_t \ys}\frac {1} {\yt^T \overline{\Q}_t f(\xt)} {\overline{\Q}_t^T\yt}\\
&= \frac {\partial f(\xt)} {\partial \xt}~\sum_{\yt}  {\overline{\Q}_t^T}\frac {1} {\yt^T \overline{\Q}_t f(\xt)}{\yt}{\yt^T \overline{\Q}_t \ys}\\
&= \frac {\partial f(\xt)} {\partial \xt} \overline{\Q}_t^T~\left ( \sum_{\yt}  \frac {1} {\yt^T \overline{\Q}_t f(\xt)}\yt\yt^T\right ) \overline{\Q}_t \ys\\
&= \frac {\partial f(\xt)} {\partial \xt} \overline{\Q}_t^T~\left ( \sum_{c=1}^{C+1}  \frac {1} {\mathbf{e}_c^T \overline{\Q}_t f(\xt)}\mathbf{e}_c\mathbf{e}_c^T\right ) \overline{\Q}_t \ys\\
&= \frac {\partial f(\xt)} {\partial \xt} \overline{\Q}_t^T~\mathbf{D}~ \overline{\Q}_t \ys,
\end{align}
where $\mathbf{D}$ is a diagonal matrix. Since the class corresponding to the absorbing state, \ie, class ${C+1}$ does not exist in the original dataset, we have $q(\ys=\mathbf{e}_{C+1})=0 \Leftrightarrow [f(\xt)]_{C+1} = 0$. Therefore, $[\mathbf{D}]_{cc} = ((1-\gamma_t)[f(\xt)]_c)^{-1}$ when $c \neq C+1$, otherwise, $\gamma_t^{-1}$. As $\ys$ is a one-hot vector,
\begin{align}
\EE_{q(\yt|\ys)}[\nabla_{\xt}\log q(\yt|\xt)] &= \frac {\partial f(\xt)} {\partial \xt} \left( (1-\gamma_t)\frac {1} {\ys^Tf(\xt)}\ys + \gamma_t (\mathbf{1}+\frac{1-\gamma_t}{\gamma_t}\mathbf{e}_{C+1}) \right) \\
&= (1-\gamma_t)\frac {1} {\ys^Tf(\xt)}\frac {\partial f(\xt)} {\partial \xt}\ys  + \gamma_t \frac {\partial f(\xt)} {\partial \xt}(\mathbf{1}+\frac{1-\gamma_t}{\gamma_t}\mathbf{e}_{C+1})\\
&= (1-\gamma_t)\nabla_{\xt}\log q(\ys|\xt) + \gamma_t \nabla_{\xt} (\mathbf{1}+\frac{1-\gamma_t}{\gamma_t}\mathbf{e}_{C+1})^{T}f(\xt) \\
&= (1-\gamma_t)\nabla_{\xt}\log q(\ys|\xt) + \gamma_t \nabla_{\xt} 1 \\
&= (1-\gamma_t)\nabla_{\xt}\log q(\ys|\xt).
% \therefore \EE_{q(\yt|\ys)}[\nabla_{\xt}\log q(\yt|\xt)] 
% &= (1-\gamma_t)\nabla_{\xt}\log q(\ys|\xt)
\end{align}
\end{proof}

\section{Algorithms}
\label{sec:supp-B}
Algorithm~\ref{alg:training} and \ref{alg:sampling} summarize the general training and sampling process of our SCDM, respectively.
While $\mathbf{m}$ has to be $\mathbf{e}_{C+1} \in \R^{C+1}$, we implemented the absorbing vector with a zero vector $\mathbf{0} \in \R^{C}$ to minimize the modification of the pretrained SIS model.
\begin{algorithm}
\def\NoNumber#1{{\def\alglinenumber##1{}\State #1}\addtocounter{ALG@line}{-1}}
\caption{Training}
\label{alg:training}
\begin{algorithmic}[1]
\REQUIRE image $\xs\in\mathbb{R}^{H\times W\times 3}$, label $\ys\in\mathbb{R}^{H\times W\times C}$, noise schedule $\alpha_{1:T},\gamma_{1:T}$, absorbing vector $\mathbf{m}\in\mathbb{R}^{C}$
\WHILE{$\text{not converged}$}
\STATE $\epsilon_\mathbf{x} \sim \mathcal{N}(\mathbf{0}, \mathbf{I})$ 
\hfill\COMMENT{// $\epsilon_\mathbf{x} \in \mathbb{R}^{H\times W \times 3}$} 
\STATE $ \forall_{i,j}  [\epsilon_\mathbf{y}]_{ij} \sim \text{Uniform} (0, 1) $
\hfill\COMMENT{// $ \epsilon_\mathbf{y} \in \mathbb{R}^{H\times W}$} 
\STATE $ t \sim \text{Uniform}(\{1, ... ,T\}) $
\STATE $\mathbf{x}_t\leftarrow \sqrt{\alpha_t}  \xs + \sqrt{1-\alpha_t}  \epsilon_\mathbf{x}$
\STATE $\forall_{i,j} [\mathbf{y}_t]_{ij} \leftarrow [\mathbf{y}_0]_{ij} \text{ if } [\epsilon_\mathbf{y}]_{ij} \ge \gamma_t \text{ else } \mathbf{m} $
\STATE $\text{Take a gradient descent step on }$
\nonumber $\mathcal{L}_\text{hybrid} \left(\epsilon_\xb,\epsilon_\theta({\mathbf{x}_t}, {\mathbf{y}_t}, t),\Sigma_\theta({\mathbf{x}_t}, {\mathbf{y}_t}, t) \right) $
\ENDWHILE
\end{algorithmic}
\end{algorithm}
\begin{algorithm}
\def\NoNumber#1{{\def\alglinenumber##1{}\STATE #1}\addtocounter{ALG@line}{-1}}
\caption{Sampling}
\begin{algorithmic}[1]
\REQUIRE $\text{label } \mathbf{y}_0\in\mathbb{R}^{H\times W\times C},~\text{noise schedule }\alpha_{1:T},\gamma_{1:T},  \text{guidance scale } s, \text{absorbing vector } \mathbf{m}\in\mathbb{R}^{C}, \newline \text{extrapolation scale } w$
\STATE $ \mathbf{x}_T \sim \mathcal{N}(\mathbf{0}, \mathbf{I})$
\FOR{$t \leftarrow T, ...,1$}
\STATE $ \forall_{i,j}  [\epsilon_\mathbf{y}]_{ij} \sim \text{Uniform} (0, 1) $
\hfill \COMMENT{// $ \epsilon_\mathbf{y} \in \mathbb{R}^{H\times W}$}
\STATE $\forall_{i,j} [\mathbf{y}_t]_{ij} \leftarrow [\mathbf{y}_0]_{ij} \text{ if } [\epsilon_\mathbf{y}]_{ij} \ge \gamma_t \text{ else } \mathbf{m} $
\STATE $\mathbf{\epsilon_{\mathbf{x}}} \sim \mathcal{N}(\mathbf{0}, \mathbf{I})\text{ if }t\neq1\text{ else }\mathbf{0}$
\hfill \COMMENT{// $\epsilon_\mathbf{x} \in \mathbb{R}^{H\times W \times 3}$}
\STATE $\tilde{\epsilon}_\theta \leftarrow \epsilon_\theta(\mathbf{x}_t,\mathbf{y}_t,t) + s(\epsilon_\theta(\mathbf{x}_t,\mathbf{y}_t,t)-\epsilon_\theta(\mathbf{x}_t,\mathbf{0},t))$
\hfill\COMMENT{// Classifier-free guidance}
\STATE $\xs^{(t)} \leftarrow \frac{\xt - \sqrt{1-\alpha_t} \tilde{\epsilon}_\theta}{\sqrt{\alpha_t}}$
\hfill\COMMENT{// Reparameterize $\epsilon$-pred model to predict $\xs$}
\STATE $\xs^{(t)} \leftarrow \text{dynamic\_thresholding} (\xs^{(t)}) $
\hfill\COMMENT{// Dynamic thresholding}
\STATE $\tilde{\xb}_0^{(t)} \leftarrow \xs^{(t)} + w\left(\xs^{(t)}  - \tilde{\xb}_0^{(t+1)}\right) \text{ if } t\neq T \text{ else } \xs^{(t)}$
\hfill\COMMENT{// Extrapolation}
\STATE $\mu_t = \frac{1}{\sqrt{1-\alpha_t}} \left( \sqrt{\alpha_{t-1}} (1-\frac{\alpha_t}{\alpha_{t-1}}) \tilde{\xb}_0^{(t)} \right. \qquad \qquad \qquad$
\nonumber $\left. \qquad \qquad \qquad \qquad \qquad  + \sqrt{\frac{\alpha_t}{\alpha_{t-1}}} (1-\alpha_{t-1}) \xt \right)$
\STATE  $\xtt \leftarrow \mu_t + \Sigma_\theta(\xt, \yt, t)^{\frac{1}{2}} \epsilon_{\mathbf{x}}$
\hfill\COMMENT{// Sample $\xtt$}
\ENDFOR
\STATE \bfseries Return $\mathbf{x}_0$
\end{algorithmic}
\label{alg:sampling}
\end{algorithm}
\section{Experimental Setup}
\label{sec:supp-C}
\subsection{Metrics}
\label{sec:supp-metrics}
\subsubsection{Experiments in Section~\ref{sec:experiments} and 
\ref{sec:abl1}}
\paragraph{FID (fidelity).}
To quantitatively measure generation quality, we adopt Fréchet Inception Distance (FID) \cite{heusel2017gans} as our evaluation metrics.
FID captures the image's visual quality by comparing the distribution between real and generated images on the inception network's~\cite{szegedy2017inception} feature space.
\paragraph{mIoU (semantic correspondence).}
Additionally, following previous works, we assess the alignment of the synthesis results with ground truth semantic maps by using off-the-shelf pretrained segmentation networks and report mean Intersection-over-Union~(mIoU).
We feed the sampled images to the pre-trained, off-the-shelf semantic segmentation networks: U-Net~\cite{lee2020maskgan,ronneberger2015u} for CelebAMask-HQ, ViT-Adapter-S~\cite{chen2022vision} with UperNet~\cite{xiao2018unified} for ADE20K, and DeepLabV2~\cite{chen2015semantic} for COCO-Stuff.
For a fair comparison, we tried to measure and reproduce the mIoU of all baseline samples.
For the standard SIS setting, when samples or checkpoints were not publicly available, we report the results from \cite{tang2023edge} and denote them with `*' in the table.
\paragraph{LPIPS (diversity).}
In addition to FID and mIoU, we adopt Learned Perceptual Image Patch Similarity (LPIPS)~\cite{zhang2018unreasonable} as our generation diversity metrics in Section~\ref{sec:exp_results} and compare the average distance between multi-modal synthesis results.
Specifically, following~\cite{tan2021diverse}, we sample 10 different images for each semantic map, compute the pairwise LPIPS distance, and average over the label maps.
\subsubsection{Experiments in Section~\ref{sec:abl2}}
\paragraph{LPIPS.} In Section~\ref{sec:abl2}, we adopt LPIPS to compare similarity, as LPIPS essentially measures the similarity of two different images.
Therefore, lower LPIPS is better in our experimental setting.
We calculate LPIPS by comparing images generated with clean and noisy labels, and averaging the results.
\paragraph{PSNR.} Peak Signal-to-Noise Ratio (PSNR) is a metric that evaluates the reconstruction quality, \ie, higher PSNR is better.
We calculate PSNR by comparing images generated with clean and noisy labels, and averaging the results.
\paragraph{SSIM.} Structural Similarity Index Measure (SSIM)~\cite{SSIM} is a metric to evaluate the similarity of a pair of images, \ie, higher SSIM is better.
We calculate SSIM by comparing images generated with clean and noisy labels, and averaging the results.
\paragraph{FID.} In Section~\ref{sec:abl2}, we use FID to compare the distribution between the set of generated images conditioned on clean labels and the set of generated images conditioned on noisy labels.
In our experimental setting, lower FID indicates that the two distributions are closer, which is better.
\subsection{Baselines}
\label{sec:supp-baselines}
We compare our model with GAN-based methods such as pix2pixHD~\cite{wang2018high}, SPADE~\cite{park2019semantic}, DAGAN~\cite{tang2020dual}, SCGAN~\cite{wang2021image}, CLADE~\cite{tan2021efficient}, CC-FPSE~\cite{liu2019learning}, GroupDNet~\cite{zhu2020semantically}, INADE~\cite{tan2021diverse}, OASIS~\cite{sushko2020you}, RESAIL~\cite{shi2022retrieval}, SAFM~\cite{lv2022semantic}, ECGAN~\cite{tang2023edge}, ECGAN++~\cite{tang2023edge2}, and Unconditional Generator with Semantic Mapper~\cite{chae2024semantic}.
We also evaluate our method with diffusion-based approaches, SDM~\cite{wang2022semantic}, LDM~\cite{rombach2022high}, and PITI~\cite{wang2022pretraining}.
Although FLIS~\cite{xue2023freestyle} and ControlNet~\cite{zhang2023adding} do tackle SIS with diffusion models, they utilize extra text inputs (captions) with the semantic maps.
For a fair comparison, we did not include FLIS and ControlNet as baselines.

As the pretrained weights of pix2pixHD and PITI on ADE20K are not publicly available, we excluded them from our noisy SIS experiments on ADE20K.
Additionally, the pretrained weights and the full code of RESAIL, ECGAN, and Unconditional Generator are not available in public, therefore we did not include them in our noisy SIS experiments baselines as well.

For the standard SIS setting, we report all the results evaluated with the uploaded samples or checkpoints except when they are unavailable.
For PITI, as the COCO-Stuff pretrained weights were only available, we sampled images with the weights and reported the LPIPS and mIoU scores.
For CelebAMask-HQ results of SDM, we report the reproduced results.
Although LDM did not include SIS experiments on benchmark datasets in their official paper, we trained their model with their SIS training configurations and reported the results in both noisy and standard SIS experiments, for more comparison with DM-based approaches.
\subsection{Implementation Details}
\label{sec:supp-implementation}
We followed the overall architecture of SDM~\cite{wang2022semantic}, a diffusion model for semantic image synthesis.
Specifically, we embed the condition into the U-Net decoder with SPADE~\cite{park2019semantic} and constructed a similar U-Net structure, where the architecture details are publicly available in their GitHub repository.
SDM trained their model with two stages: (1) pre-trained their conditional diffusion model and (2) fine-tuned the pre-trained model by randomly dropping out the semantic label maps in order to use classifier-free guidance in sampling.
We trained our SCDM starting from their pre-trained weights to reduce the training time.
We trained our model with 4 NVIDIA RTX A6000 GPUs for 1-2 days.
Image sampling and evaluations are conducted on a server with 8 NVIDIA RTX 3090 GPUs.

For the hyperparameters, we used $\lambda=0.001$ for our hybrid loss~\cite{nichol2021improved}, classifier-free guidance~\cite{ho2022classifier} scale $s=0.5$, $20\%$ of drop rate for the SIS experiments on three datasets, noise schedule hyperparameter $\eta=1$, dynamic thresholding~\cite{saharia2022photorealistic} percentile of $0.95$, and the extrapolation scale of $w=0.8$.
Except for the ablation study (in Table~\ref{tab:ablation}), we didn't use extrapolation for experiments with 1000 sampling steps, as extrapolation is designed to enhance small-step generation.
We also utilized exponential moving average~(EMA) with 0.9999 decay, and AdamW~\cite{loshchilov2019decoupled} optimizer.
We employed instance labels of CelebAMask-HQ and COCO-Stuff to produce instance edge maps and used them as additional input following SDM.

Furthermore, for the class-wise noise schedule, we clamped  $\phi_c$ to be at least 1, to down the diffusion of rare classes without 
speeding up the transition of frequent classes.
We set the class-wise noise schedules of the `unlabeled' class of ADE20K and COCO-Stuff to be the uniform schedule, \ie, $\gamma_{t,c} = t/T$ for $c=$`unlabeled'.
In addition, we implemented the calculation of $\gamma_{t,c}$ using $t-1$ instead of $t$ in the code to ensure that the assumption of Lemma 1 is satisfied.

For our experiments on ADE20K, we modified the calculation of $\psi_c  \phi_c$ values of the class-wise noise schedule with an empirically chosen scale factor.
This is because ADE20K has numerous classes and high diversity in the dataset compared to the number of training data, while CelebAMask-HQ has relatively low diversity and COCO-Stuff has a large number of training data.
Specifically, we modified $\psi_c  \phi_c$ as:
\begin{equation}
    \psi_c  \phi_c:= \lambda~\Pr(\yb_{ij}=c)^{-1} \log \left( \Pr(\xb\in \mathcal{X}_c)^{-1}\right),
\end{equation}
where $\lambda$ is the scale factor~($\approx 0.278$) making the smallest value of $\psi_c  \phi_c$ converge to 1, \ie, $\gamma_{t,c} = t/T$. 

\subsection{Efficient Trajectory Representation}
\label{sec:supp-trajectory}
The generation process of SCDM utilizes $\yb_T, ..., \yb_1$ at each timestep $t=T, ... ,1$.
As we have not defined the reverse process of the labels, we use the label forward process $q(\yb_{1:T}|\ys)$, which sequentially constructs $\yb_1, ..., \yb_T$ from $\ys$.
Consequently, $\yb_{1:T}$ needs to be \textit{cached} during sampling in order to be accessed from $t=T$ to $t=1$.
However, this results in a significant memory consumption, posing a challenge to our sampling process.
To address this issue, we leverage $\gamma_{t,c}$ from Eq~\eqref{eq:marginal}.
Specifically, we represent the entire trajectory $\yb_{1:T}$ with a single matrix $\mathbf{U} \in \mathbb{R}^{H\times W}$, where $\mathbf{U}_{ij}$ denotes the `timestep' that $(i,j)$-th pixel is masked.
Suppose that $(i,j)$-th pixel corresponds to class $c$.
Since $\gamma_{t,c}$ can be expressed as a strictly monotonic function $\gamma_c(t)$, we can easily sample $\mathbf{U}_{ij}$ using the inverse CDF method on $\gamma_c^{-1}(u)$, where $u\sim \text{Uniform}(0,1)$.

\section{Detailed Explanations on $\psi_c$ and $\phi_c$}
\label{sec:supp-psiphi}
\begin{figure}[t!]
\centering
\includegraphics[width=0.5\columnwidth]{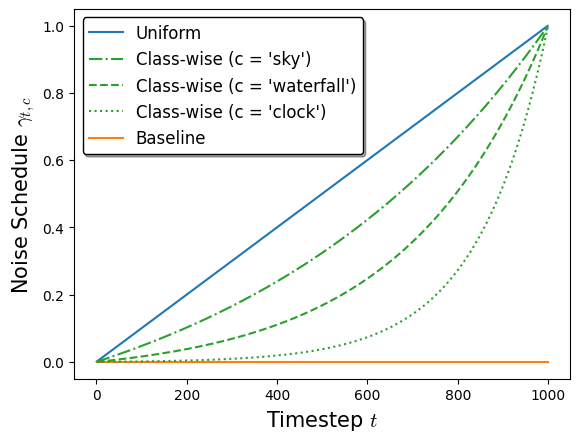}
\caption{\textbf{Visualization of $\gamma_{t,c}$ throughout diffusion in the baseline, linear and uniform, and class-wise noise schedule.} $\gamma_{t,c}$ indicates the probability of a label $c$ has transitioned to the absorbing state until timestep $t$.
In our class-wise schedule, small and rare objects tend to be intact relatively longer in the diffusion process, \eg, clock.}
\label{fig:graph}
\end{figure}
In this section, we elaborate on $\psi_c$ and $\phi_c$ as introduced in Section~\ref{sec:transition-matrix-and-noise-scheduling} and provide some visual aid for $\gamma_{t,c}$.
$\psi_c$ accounts for the area ($\approx$ object size) of class $c$.
Eq.~\eqref{eq:itf} defines $\psi_c$, which we estimated with training data.
$\psi_c$ ensures that classes with smaller average areas transition to the masked state later.
Specifically, the term is defined as the inverse of the ratio of average total image pixels to the average class pixels.
As the average area of class $c$ decreases, the $\psi_c$ value increases.
On the other hand, $\phi_c$ ensures a slow diffusion (to the masked state) of classes that rarely appear in the dataset.
It is defined in Eq.~\eqref{eq:idf}, and it can be regarded as the inverse frequency of the class in training images.
A larger $\phi_c$ value corresponds to a class that rarely appears in the dataset.
In Figure~\ref{fig:graph}, we visualize $\gamma_{t,c}$ with different noise schedules.
The class-wise schedule of class `clock' ensures that the class is diffused at a relatively later time, as its $\psi_c \phi_c$ value is large.
\section{Limitations}
\label{sec:limitations}
One possible limitation of this work is that as the method aims to generate clean images from the noisy labels, the user's intention might be ignored, \ie, the user might actually want to generate noisy images. 
Controlling faithfulness to the semantic maps could be one possible future direction, which we leave as future work.
Another limitation could be that our method cannot dynamically learn the optimal noise schedules for labels.
If the proposed method can learn the optimal noise schedule, it will improve the flexibility and power of semantic image synthesis.
\section{ADE20K Dataset Annotation Examples}
\label{sec:supp-D}
In this section, we provide examples of erroneous/inconsistent annotations.

\paragraph{Jagged edges (DS).}
We observe that some images have jagged edges as label maps given in Figure~\ref{fig:supp-jagged}.
This could be the result of erroneous annotations but also could occur due to the coarse and low resolution of the images.
ADE20K dataset contains images of various sizes.
Similarly, the size of user inputs can be diverse in real-world applications.
\paragraph{Incomplete masks (Edge).}
Figure~\ref{fig:supp-incomplete} shows some examples of incomplete masks, where the pixels on the boundaries of instances were left as `unlabeled'. 
This usually occurs due to their inherent ambiguity.
\paragraph{Inconsistent annotation (\eg, umbrella).}
Similar to Figure~\ref{fig:miou}, Figure~\ref{fig:supp_umbrella} shows training images~(first row) and their annotations~(second row).
Some umbrellas are correctly annotated with the shafts~(Figure~\ref{fig:supp_umbrella_yes}), while others are annotated without the shafts and handles~(Figure~\ref{fig:supp_umbrella_no}).
Trained with these inconsistent labels, baselines did not synthesize the shaft for the parasol (class `umbrella') in Figure~\ref{fig:miou}, whereas our SCDM successfully generated the shafts leveraging our Label Diffusion.
\section{Additional Experimental Results}
\label{sec:supp-E}
\subsection{Standard SIS Setting} 
\label{sec:exp_results}
\subsubsection{Generation quality comparison.}
Figure~\ref{fig:celeba} presents the qualitative comparison with other SIS models and our method on CelebAMask-HQ.
Remarkably, our method generates realistic images even for an intricate semantic map with fluttering hair over the face, while others struggle with generating the occluded eye naturally. 

As Figure~\ref{fig:ade20k} illustrates the qualitative results on ADE20K, our method demonstrates its advantage over previous methods in generating details, and convincing images within the given semantics.
Especially, our SCDM generates a more realistic and clear image of the waterfall than the baselines.
In addition, our approach more naturally depicts the rocks visible through the water, compared to the baselines.
For ADE20K, we used slightly modified $\psi_c\phi_c$ with some scale factors, and the details are in Appendix~\ref{sec:supp-implementation}.

Lastly, on COCO-Stuff our method outperforms the state-of-the-art methods in FID and LPIPS.
The qualitative comparison with others and ours is given in Figure~\ref{fig:coco}.
Ours synthesized a more realistic and clearer image of a clock tower, by successfully generating the numbers and hands of the clock.
Furthermore, ours generated a more natural image of the tree branches over the clock tower.

\begin{figure}[t!]
\centering
\includegraphics[width=0.9\columnwidth]{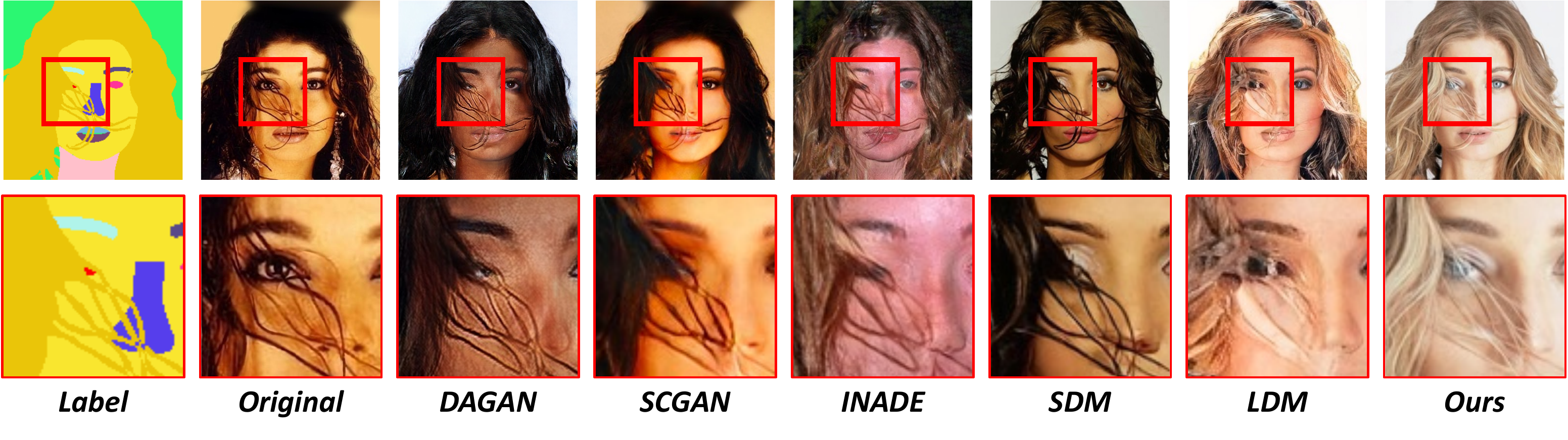}
\caption{ \textbf{Qualitative results on CelebAMask-HQ.} Ours successfully catches the fine-grained details and generates realistic images conditioning on challenging semantic labels with highly occluded objects (\eg, eye covered by hairs). }
\label{fig:celeba}
\end{figure}
\begin{figure}[t!]
\centering
\includegraphics[width=0.8\columnwidth]{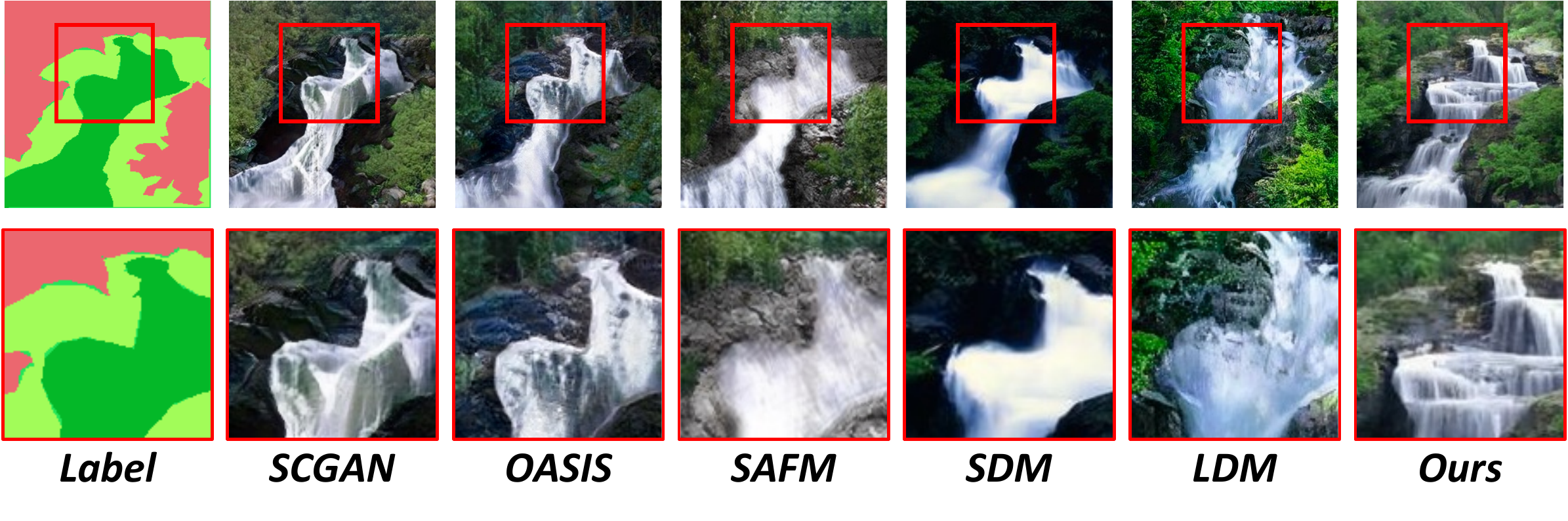}
\caption{ \textbf{Qualitative results on ADE20K.} Ours is capable of synthesizing realistic and clear images, as our generated images show more depth and clearer results of the waterfall.}
\label{fig:ade20k}
\end{figure}
\begin{figure}[t!]
\centering
\includegraphics[width=0.8\columnwidth]{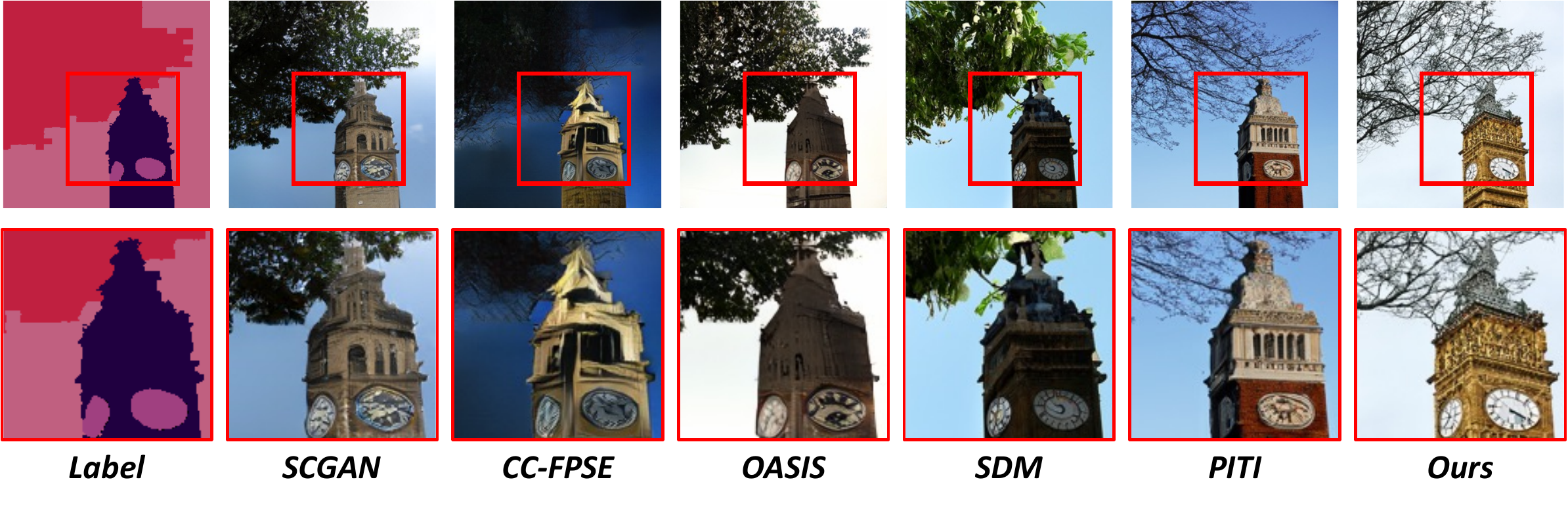}
\caption{ \textbf{Qualitative results on COCO.} 
Our SCDM clearly captures the details (\eg, the hands on the clock) and naturally depicts the branches over the clock tower.
}
\label{fig:coco}
\end{figure}
\begin{figure}[t!]
\centering
\includegraphics[width=0.6\columnwidth]{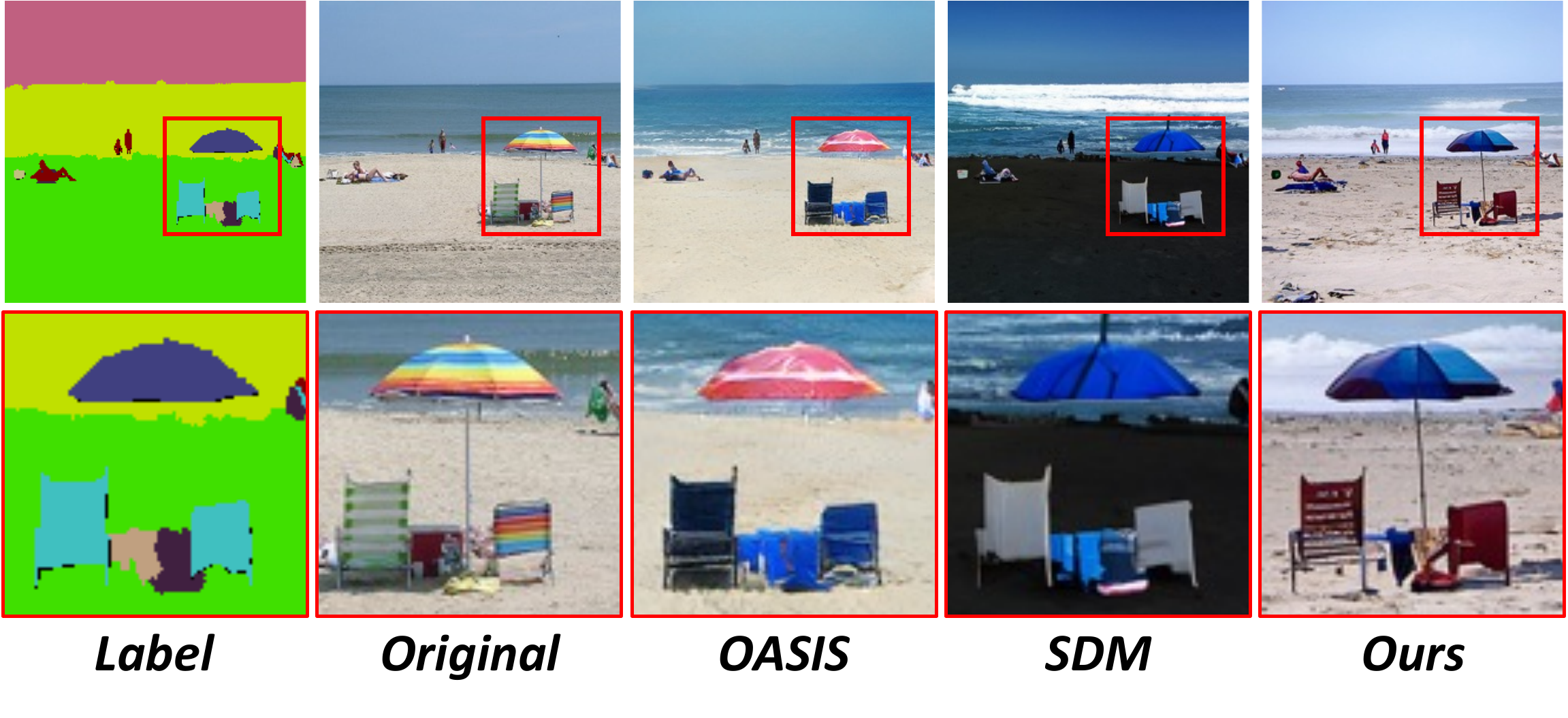}
\caption{ \textbf{Robustness to inconsistent semantic labels.} 
Our SCDM exhibits robustness to inconsistent semantic labels. 
Although the parasol~(class `umbrella') in the original image has a shaft, it is not accurately annotated in the semantic map.
Interestingly, ours synthesized a natural image with the missing shaft whereas all baselines failed to generate it. }
\label{fig:miou}
\end{figure}
\subsubsection{Analysis on semantic correspondence.}
We evaluate the semantic correspondence using pre-trained segmentation networks and compare our method with baselines on mIoU, and the results are in Table~\ref{tab:sis-main}.
SCDM shows better performance than the state-of-the-art on CelebAMask-HQ, while for the other two datasets, diffusion-based approaches including ours show comparatively weaker performance than adversarial methods such as SAFM or ECGAN.
We believe that one of the reasons for this could be the erroneous annotation of images.
For example, `umbrella' is often annotated without a shaft and a handle~(examples are given in Appendix~\ref{sec:supp-D}). 
In Figure~\ref{fig:miou}, baselines trained with these images did not synthesize the shaft for the given semantic map of a parasol (class `umbrella') without a shaft, while ours did.
Although ours is more realistic, this can result in lower mIoU scores.
We also suspect that the performance of the off-the-shelf segmentation models used for mIoU evaluation is not very robust. 
We ran the segmentation model on ground-truth (\ie, real and clean) images and it shows 
a poor mIoU with semantic labels (43.1 on ADE20K and 34.7 on COCO).
Additionally, it is noteworthy that baselines with high mIoU scores such as RESAIL, SAFM, and ECGAN leverage `pretrained' segmentation models in their semantic alignment losses (denoted by `Seg' in Table~\ref{tab:sis-main}).
In contrast, ours does not rely on such explicit guidance for correspondence.
\subsection{Extrapolation Hyperparameter Search}
\label{sec:supp-extr-w}
\begin{table}[t]
\caption{\textbf{Extrapolation hyperparameter search on \textit{ADE20K}.} The samples are generated with 25 sampling steps.
    }
    \label{tab:supp-extr}
    \vskip 0.15in
    \begin{center}
    \setlength{\tabcolsep}{3pt}
    \resizebox{0.15\columnwidth}{!}{
    \begin{tabular}{c|c}
        \toprule
        \multirow{1}{*}{$w$} & \multicolumn{1}{c}{\textbf{FID ($\downarrow$) }}\\
        \midrule
        \midrule
        0.4 & 33.0 \\
        0.5 & 31.6 \\
        0.6 & 30.3  \\
        0.7 & 29.2  \\
        0.8 & \textbf{27.7} \\
        0.9 & 30.9 \\
        \bottomrule
    \end{tabular}
    }
    \end{center}
\end{table}
For the extrapolation hyperparameter $w$, we searched for values in [0.4, 0.5, 0.6, 0.7, 0.8, 0.9], following a similar protocol of \cite{ho2022classifier}.
The search results are in Table~\ref{tab:supp-extr}, and we chose the value $w=0.8$ in terms of the generation quality (FID).
\subsection{Visualization of Label Diffusion}
\label{sec:supp-vis}
We show the intermediate steps of our Label Diffusion as a visual aid for understanding the two different noise schedules.
Figure~\ref{fig:supp_uniform} and \ref{fig:supp_classwise} illustrate the intermediate results of the diffused labels in the Label Diffusion process using the linear and uniform noise schedule and the class-wise schedule, respectively.
The first row of each subfigure shows the diffused semantic map given to the model during the discrete diffusion process, while the second row magnifies and displays the label of a clock on a cabinet from the first row.
The $\psi_c \phi_c$ values of the `clock' and `cabinet' classes are $651.3$ and $17.3$, respectively.
As the class `clock' occupies relatively small areas in the dataset, the class-wise schedule makes the clock diffused slowly, and denoised fastly during the generation process. 
Meanwhile, in the uniform noise schedule, the `clock' class is perturbed faster than in the class-wise schedule~(Figure~\ref{fig:supp_classwise}), because the uniform schedule evenly diffuses all class labels in the label map.
\subsection{Effect of Class-wise Noise Schedule}
\label{sec:supp-classwise}
\begin{table}[t]
\caption{\textbf{Quantitative noise schedule comparison on semantic correspondence (mIoU ($\uparrow$)).}
    The results are sampled over 1000 sampling steps.
    }
    \label{tab:supp-coco}
    \vskip 0.15in
    \begin{center}
    \setlength{\tabcolsep}{3pt}
    \resizebox{0.6\columnwidth}{!}{
    \begin{tabular}{l| c | c | c}
        \toprule
        \textbf{Noise Schedule} & \textbf{CelebAMask-HQ} & \textbf{ADE20K} & \textbf{COCO-Stuff} \\
        \midrule
        \midrule
        \textbf{Linear \& Uniform} & 76.8 & 43.0 & 36.8 \\
        \midrule
        \textbf{Class-wise} & \textbf{77.2} & \textbf{49.4} & \textbf{38.1} \\
        \bottomrule
    \end{tabular}
    }
    \end{center}
\end{table}
In this section, we show the quantitative comparison of the class-wise schedule and the linear and uniform schedule, continuing the discussion from Section~\ref{sec:noise_schedule_analysis} of the main paper.
The results are given in Table~\ref{tab:supp-coco}, on three (original) benchmark datasets.
The results are sampled with 1000 sampling steps, without extrapolation.
As demonstrated by the mIoU gains, the class-wise noise schedule increases the semantic correspondence of the generated images.
Additionally, we provide the qualitative comparison on COCO-Stuff in Figure~\ref{fig:supp-coco-classwise}.
\subsection{Further Discussion and Analysis on Class Guidance}
\label{sec:supp-cfg-scale}
In Section~\ref{sec:discussion}, we analyzed Label Diffusion regarding the conditional score that our method approximates.
Although our method does not have the same guidance as the baseline with fixed labels and time-dependent scaling, we further assess the effect of scaling and compare it with our method.
Specifically, we modify the fixed classifier-free guidance scale through scheduling, \ie, $s$ to $s(1-\gamma_t)$ in Eq.~\eqref{eq:cfg}, resulting in comparable fidelity with the baseline using fixed guidance (FID ($\downarrow$) of 28.6 (CFG scale scheduling) and 28.1 (baseline)).
Nevertheless, it still yields suboptimal results compared to ours (FID of 26.9).
What further distinguishes our method from CFG scale scheduling is that our absorbing state explicitly informs the model that the pixel is unconditional, facilitating the natural generation of the ambiguous parts in an image.
Some concrete examples of the ambiguous parts are given in Figure~\ref{fig:celeba} and \ref{fig:coco}, where the eye is occluded by hairs and the tree branches are over the clock tower.
\subsection{Validation of SCDM Generation Process}
\label{sec:supp-validation-gen}
\begin{table}[t]
\caption{\textbf{Sampling results with different noise schedules for generation process on \textit{ADE20K}.}
    The results are sampled from (a) SCDM trained with a class-wise noise schedule ($\eta=1$), and (b) baseline (trained without Label Diffusion, \ie, $\eta=+\infty$).
    }
    \label{tab:supp-y0}
    \vskip 0.15in
    \begin{center}
    \setlength{\tabcolsep}{3pt}
    \resizebox{0.5\columnwidth}{!}{
    \begin{tabular}{c|l|c| c | c}
        \toprule
        \textbf{Model} & \textbf{Sampling noise schedule}& \textbf{$\eta$} & \textbf{FID} & \textbf{mIoU} \\
        \midrule
        \midrule
        \multirow{3}{*}{\textbf{(a) Ours}} & {No Label Diffusion} & $+\infty$ & 30.6 & 47.8 \\
        % \midrule
        & {Linear and uniform} & $0$ & 34.5 & 28.9 \\
        % \midrule
        & {Class-wise} & $1$ & \textbf{27.7}& \textbf{48.7} \\
        \midrule
        \multirow{3}{*}{\textbf{(b) Base}} & {No Label Diffusion} & $+\infty$ & \textbf{28.1} & \textbf{48.6} \\
        & {Linear and uniform} & $0$ & 76.2 & 15.7 \\
        & {Class-wise} & $1$ & 44.0 & 43.1 \\
        \bottomrule
    \end{tabular}
    }
    \end{center}
\end{table}
In this section, we validate the SCDM generation process by showing the effect of our Label Diffusion during sampling.
Although our forward process diffuses images and labels independently~(Eq.~(9) of the main paper), our model successfully learns the joint distribution of the images and labels, and therefore generates the images dependent on the given perturbed labels.

To show this empirically, we first trained our model with the class-wise noise schedule (\ie, $\eta=1$). 
Then we sampled images with three different Label Diffusion noise schedules by controlling $\eta$: (1) $\eta=+\infty$, \ie, no Label Diffusion and using fixed $\ys$ for the entire generation process, (2) $\eta=0$, \ie, the linear and uniform noise schedule, and (3) $\eta=1$, \ie, the class-wise noise schedule that was leveraged during training.
The results are reported in Table~\ref{tab:supp-y0}(a) and we used 25 sampling steps.
As the FID scores show that the quality of the samples not generated with the trained noise schedule is worse than the images sampled with $\eta=1$, we can observe that Label Diffusion affects the generation process.

Additionally, we sampled images using the baseline model with the same noise schedules, \ie, $\eta=+\infty$, $\eta=0$, and $\eta=1$.
The results are reported in Table~\ref{tab:supp-y0}(b) and we used 1000 sampling steps (without extrapolation).
The performance deteriorates when Label Diffusion is applied to the baseline (\ie, trained without Label Diffusion).
This also indicates that Label Diffusion contributes to the generation process, validating the SCDM generation process.
\subsection{More Qualitative Results - Multimodal Generation}
\label{sec:supp-multimodal}
We provide additional generation results of our method, showing SCDM's capability of generating diverse samples.
Figure~\ref{fig:supp_diversity} shows multimodal generation results on CelebAMask-HQ, ADE20K, and COCO-Stuff (standard SIS setting).
\begin{figure}[t!]
\centering
\subfigure[\textit{CelebAMask-HQ}.]{\includegraphics[width=0.3\columnwidth]{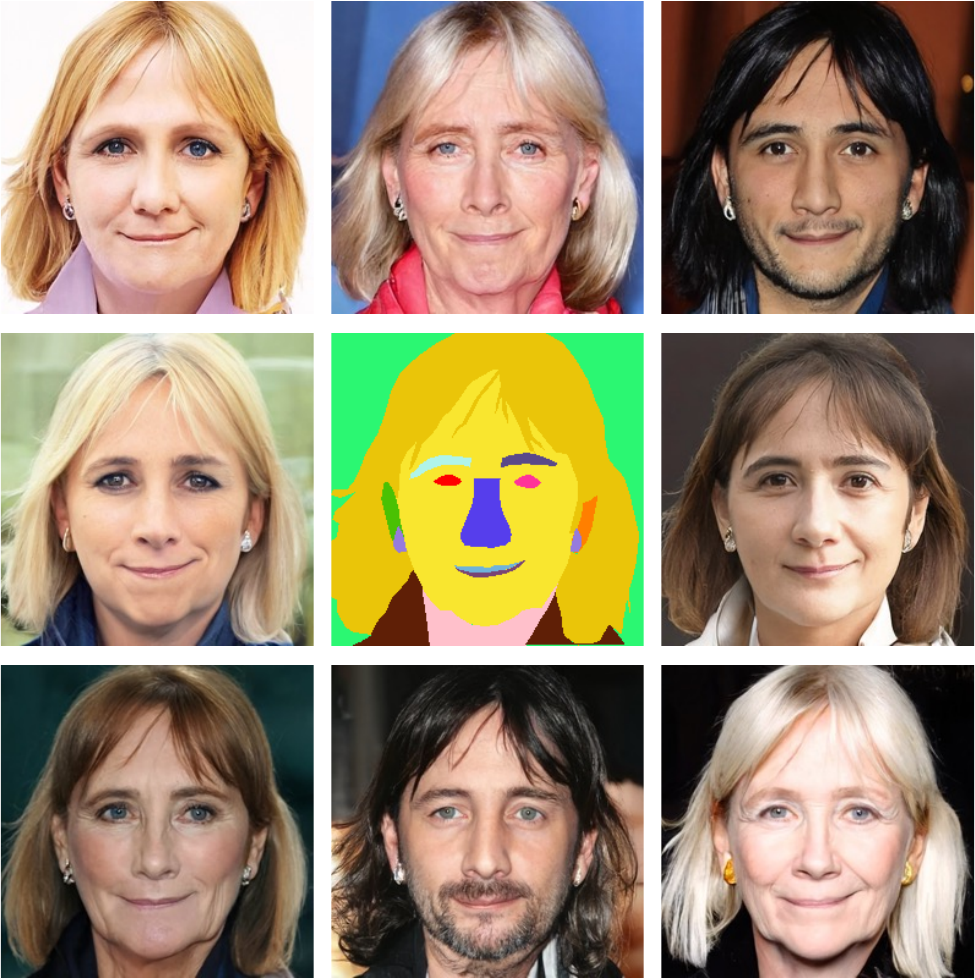}
\label{fig:supp-diversity-celeba}}
\hfill
  \subfigure[\textit{ADE20K}.]{\includegraphics[width=0.3\columnwidth]{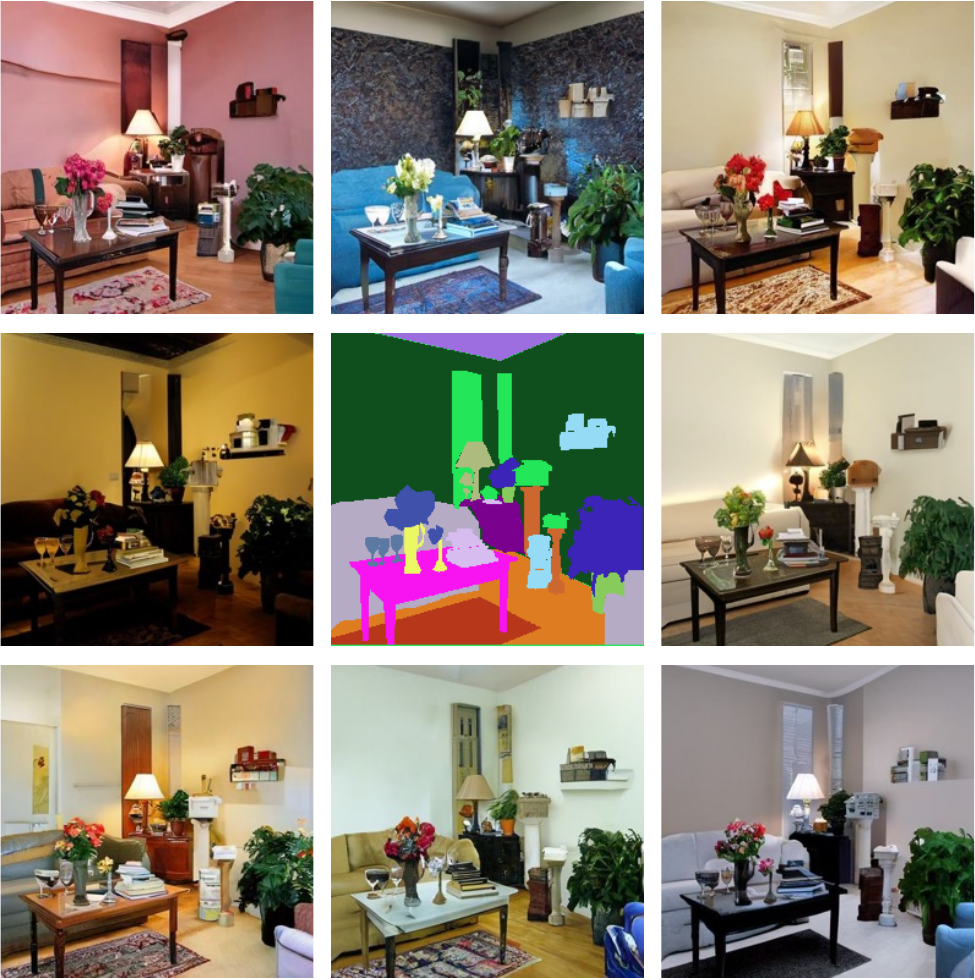}
  \label{fig:supp-diversity-ade20k}}
  \hfill
   \subfigure[\textit{COCO-Stuff}.]{\includegraphics[width=0.3\columnwidth]{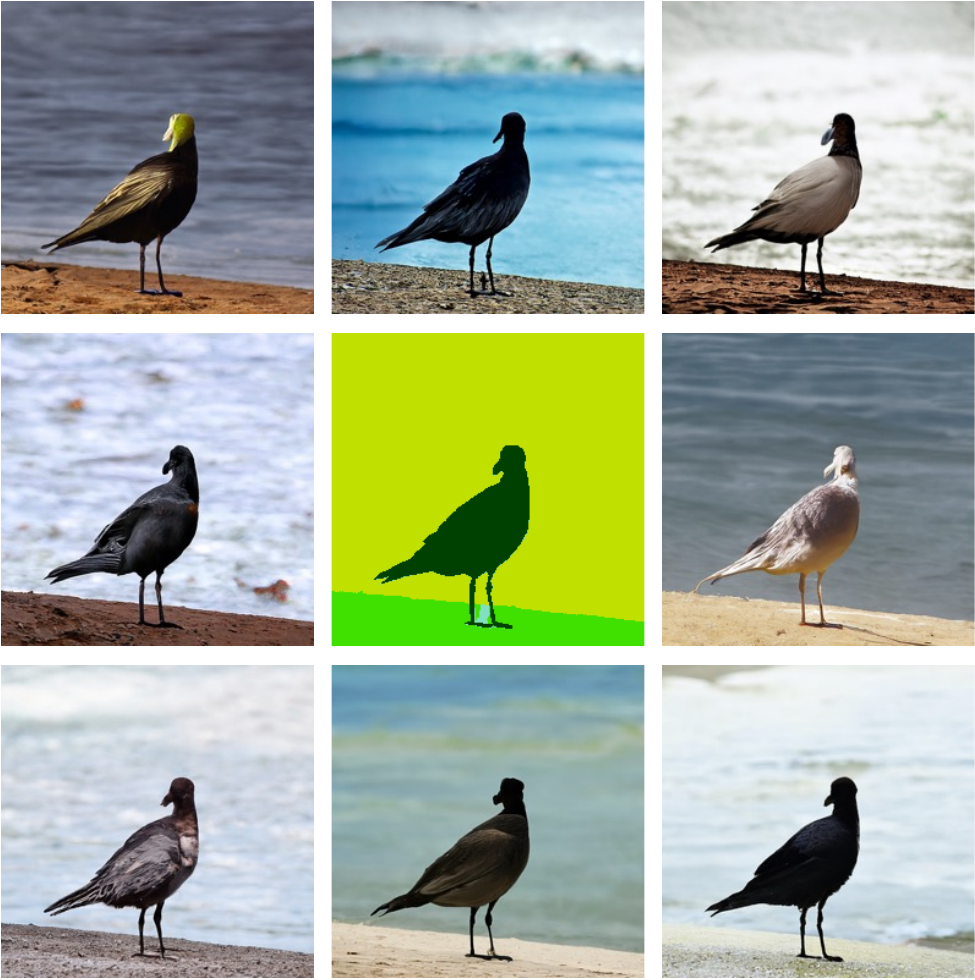}
  \label{fig:supp-diversity-coco}}
\caption{\textbf{Multimodal generation results.}
}
\label{fig:supp_diversity}
\end{figure}
\subsection{More Qualitative Results - SIS with Noisy Labels}
\label{sec:supp-more-noisy}
We provide additional generation results of our method and baselines on SIS with noisy labels. 
The results of \textbf{[DS]}, \textbf{[Edge]}, and \textbf{[Random]} are illustrated in Figure~\ref{fig:supp-qual-noisy-ds}, \ref{fig:supp-qual-noisy-edge}, and \ref{fig:supp-qual-noisy-rand}, respectively, and the experiments are on the ADE20K dataset.
\subsection{More Qualitative Results - Standard SIS Setting}
\label{sec:supp-more-standard}
We present additional generation results of our SCDM and other baselines on the standard SIS setting.
Figure \ref{fig:supp_celeba} shows the results of our method and qualitative comparisons on CelebAMask-HQ.
Figure~\ref{fig:supp_ade20k} and \ref{fig:supp_coco} show the results and qualitative comparisons on ADE20K and COCO-Stuff, respectively.
\begin{figure*}[t!]
\centering
\subfigure[Training data examples with shaft annotated.]{\includegraphics[width=0.45\textwidth]{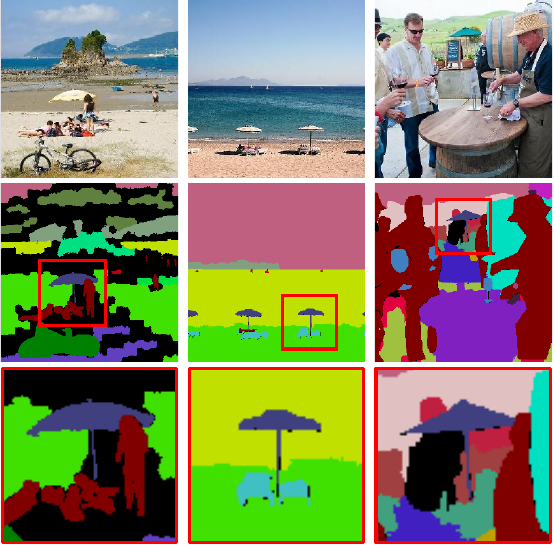}
\label{fig:supp_umbrella_yes}}
\hspace{8pt}
  \subfigure[Training data examples with shaft not annotated.]{\includegraphics[width=0.45\textwidth]{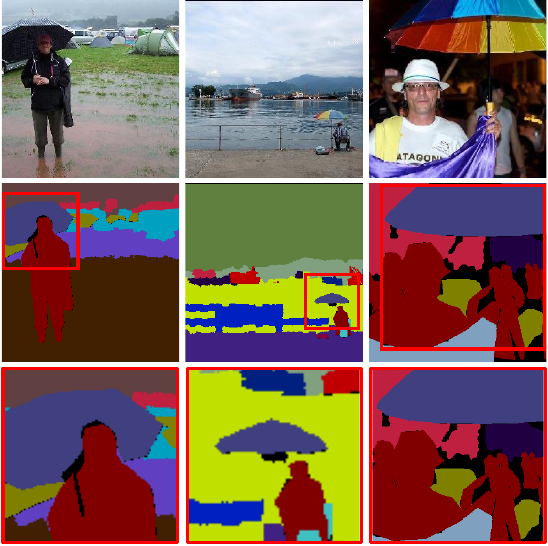}
  \label{fig:supp_umbrella_no}}
\caption{\textbf{Examples of inconsistent semantic labels (class `umbrella') in the training set.}
}
\label{fig:supp_umbrella}
\end{figure*}
\begin{figure*}[t!]
\centering
\includegraphics[width=0.7\textwidth]{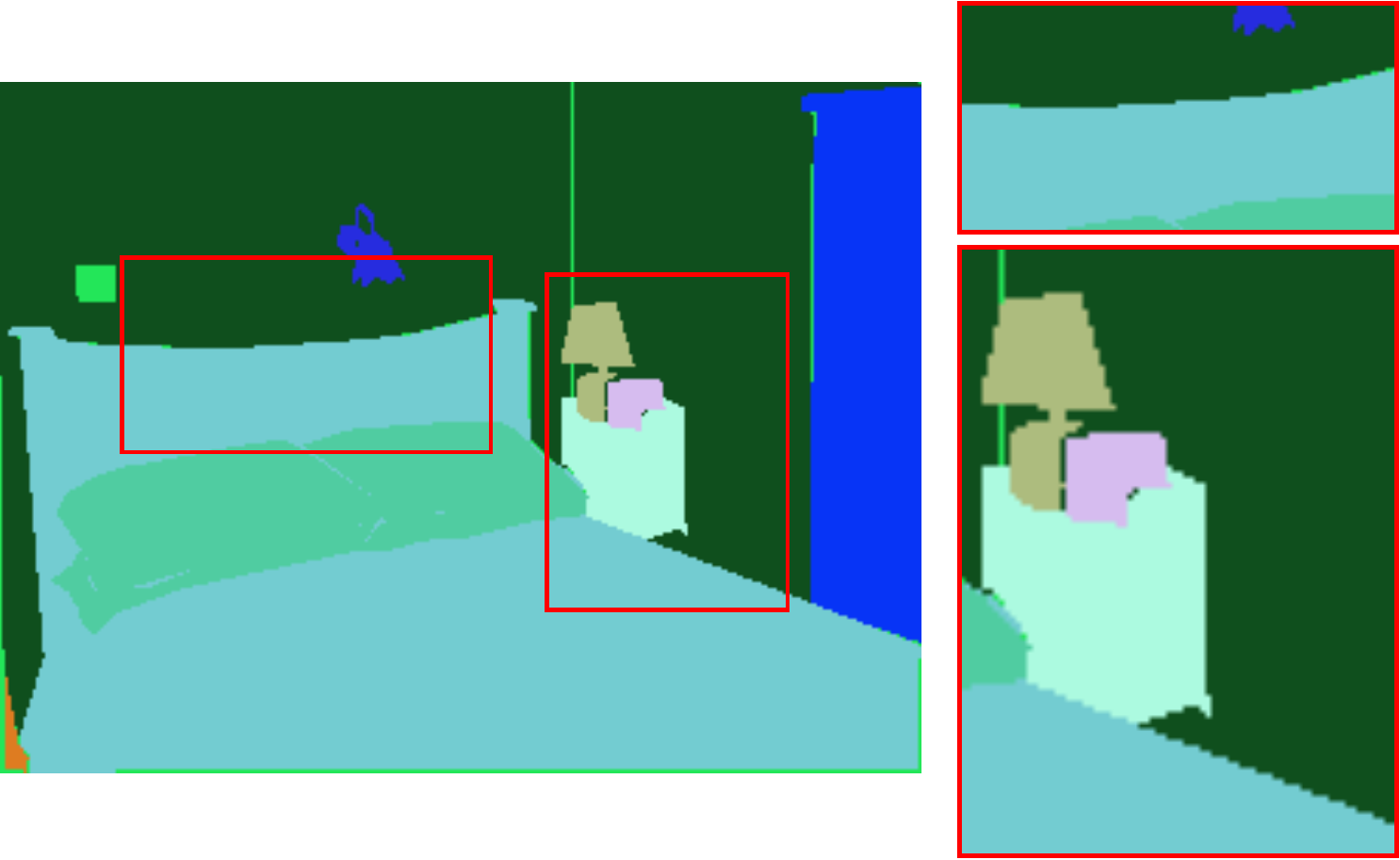}
\caption{ \textbf{Examples of jagged edges in the ADE20K dataset.}
}
\label{fig:supp-jagged}
\vspace{10pt}
\end{figure*}
\begin{figure*}[t!]
\centering
\includegraphics[width=1.0\textwidth]{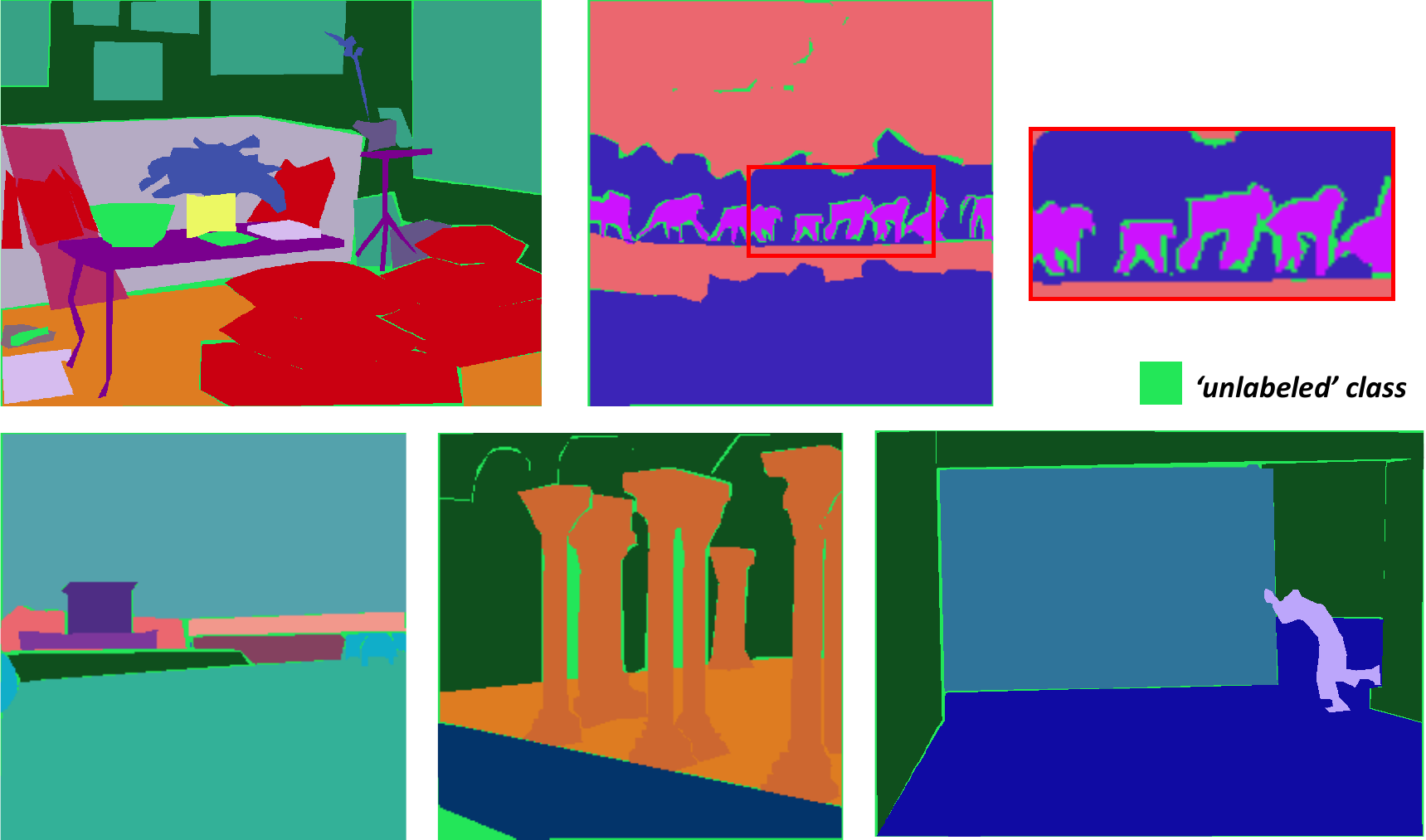}
\caption{ \textbf{Examples of incomplete masks in the \textit{ADE20K} dataset.}
}
\label{fig:supp-incomplete}
\end{figure*}
\begin{figure*}[t!]
\centering
\includegraphics[width=0.85\textwidth]{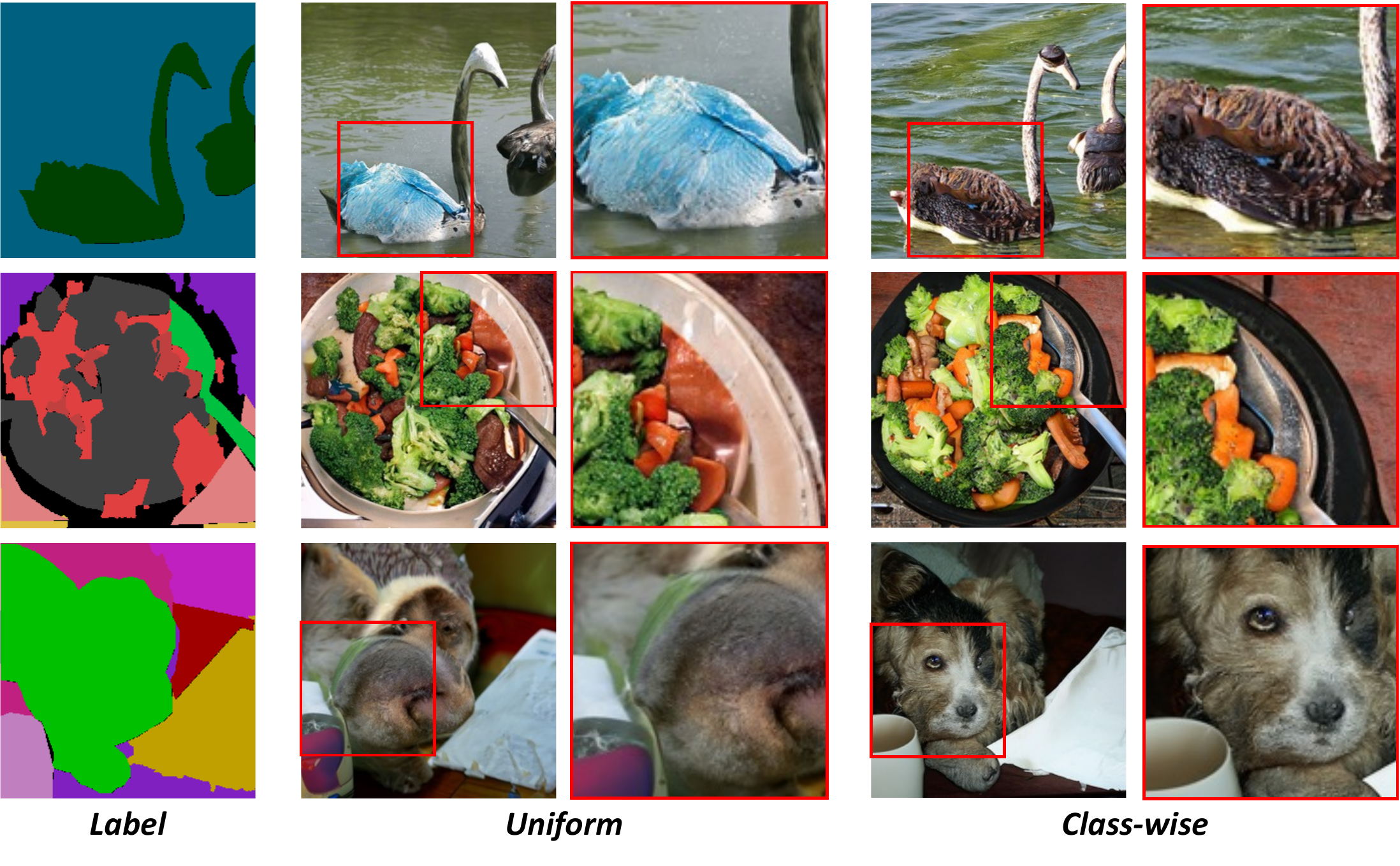}
\caption{ \textbf{Qualitative noise schedule comparison on \textit{COCO-Stuff}.} 
}
\label{fig:supp-coco-classwise}
\end{figure*}
\begin{figure*}[ht!]
\centering
\subfigure[Uniform and linear noise schedule.]{\includegraphics[width=\textwidth]{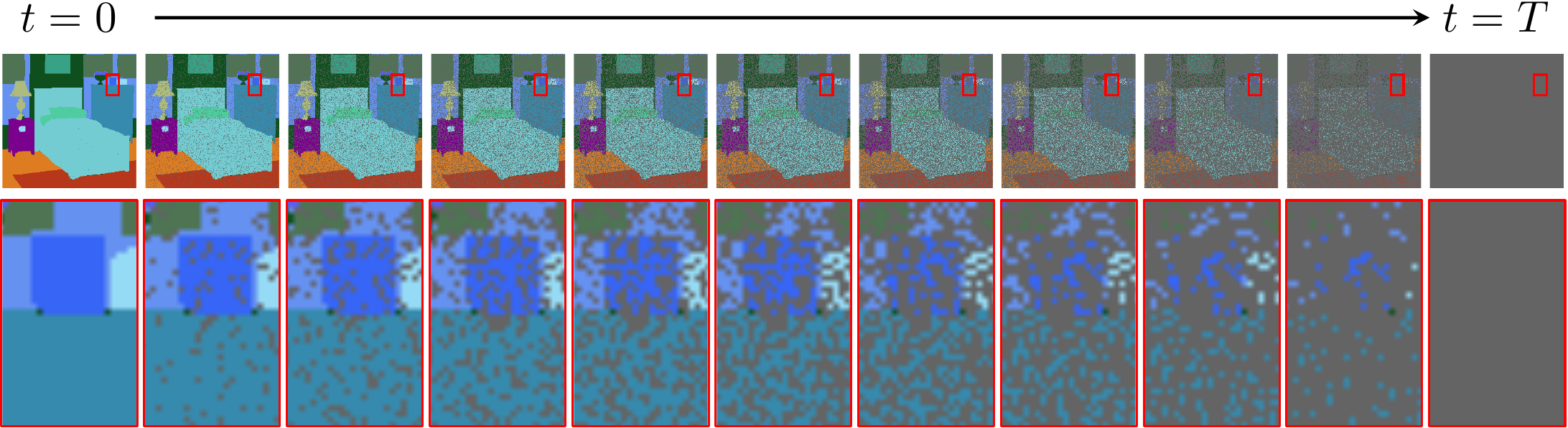}
\label{fig:supp_uniform}}
  \subfigure[Class-wise noise schedule.]{\includegraphics[width=\textwidth]{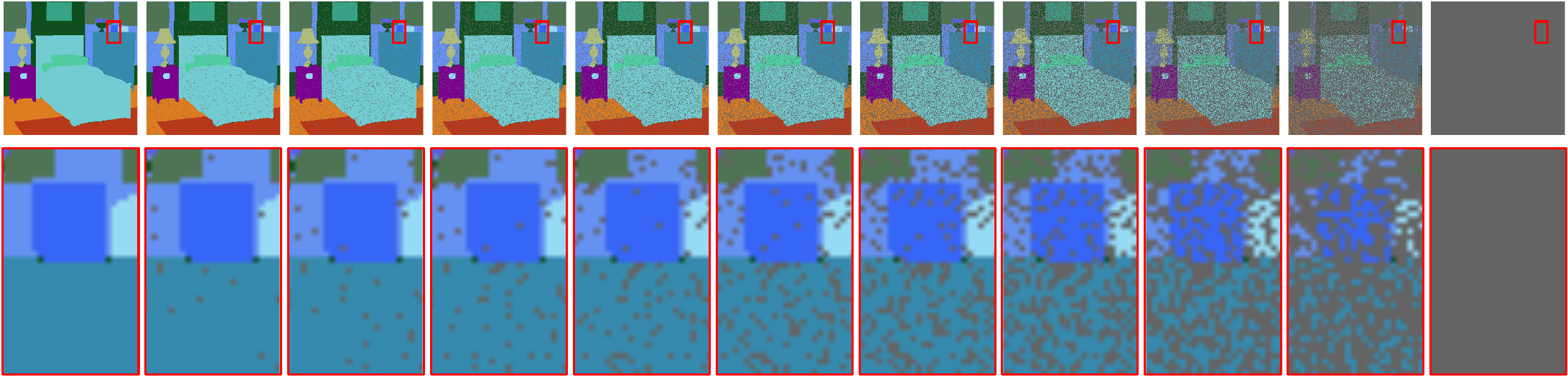}
  \label{fig:supp_classwise}}
\caption{\textbf{Visualization of Label Diffusion intermediate steps.}}
\label{fig:supp_intermediate}
\vspace{-10pt}
\end{figure*}
\begin{figure*}[t!]
\centering
\includegraphics[width=0.85\textwidth]{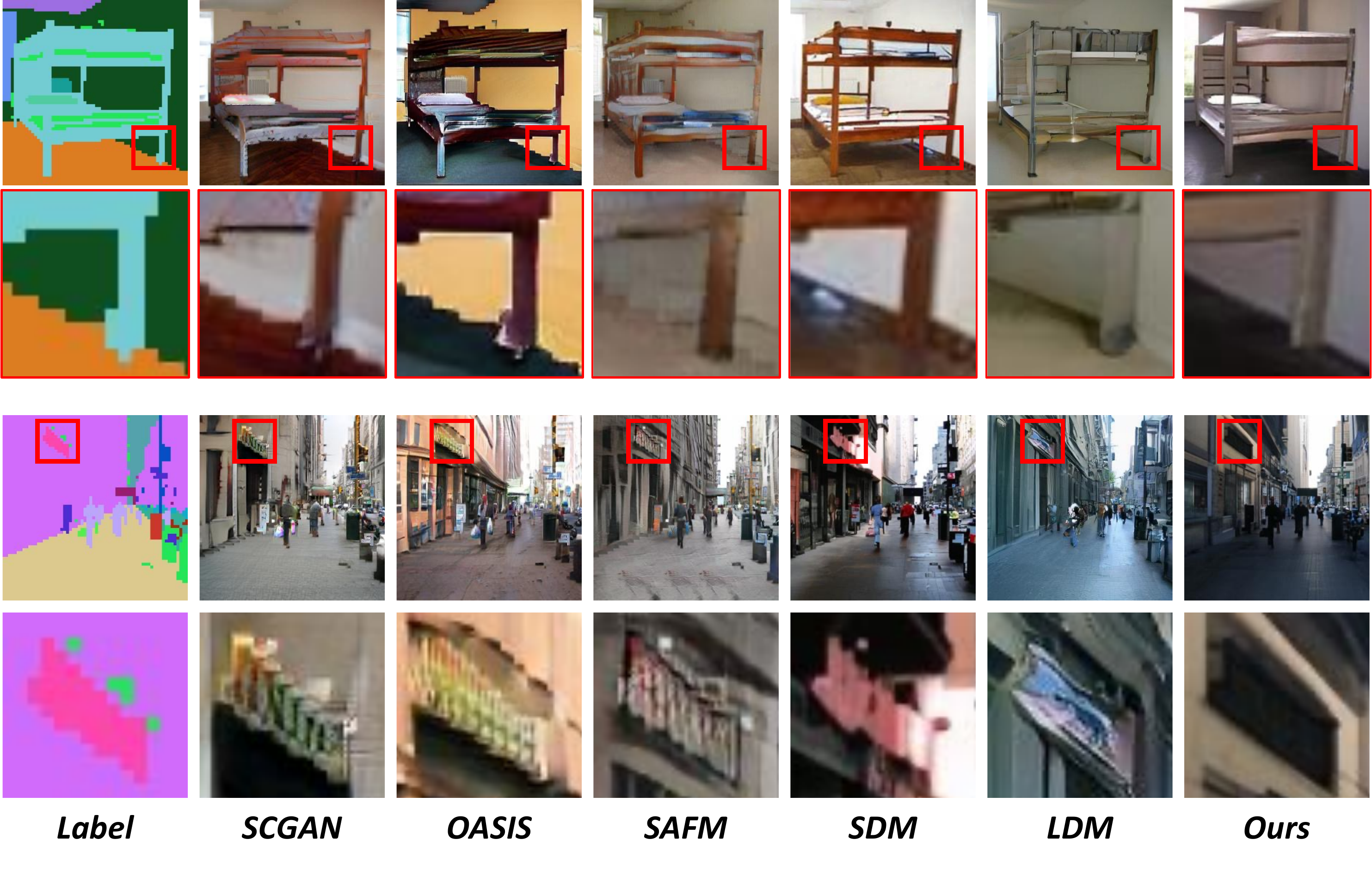}
\caption{ \textbf{More qualitative comparisons on SIS with noisy labels (DS).} The generation results are conditioned with semantic masks with jagged edges.
}
\label{fig:supp-qual-noisy-ds}
\end{figure*}
\begin{figure*}[t!]
\centering
\includegraphics[width=0.85\textwidth]{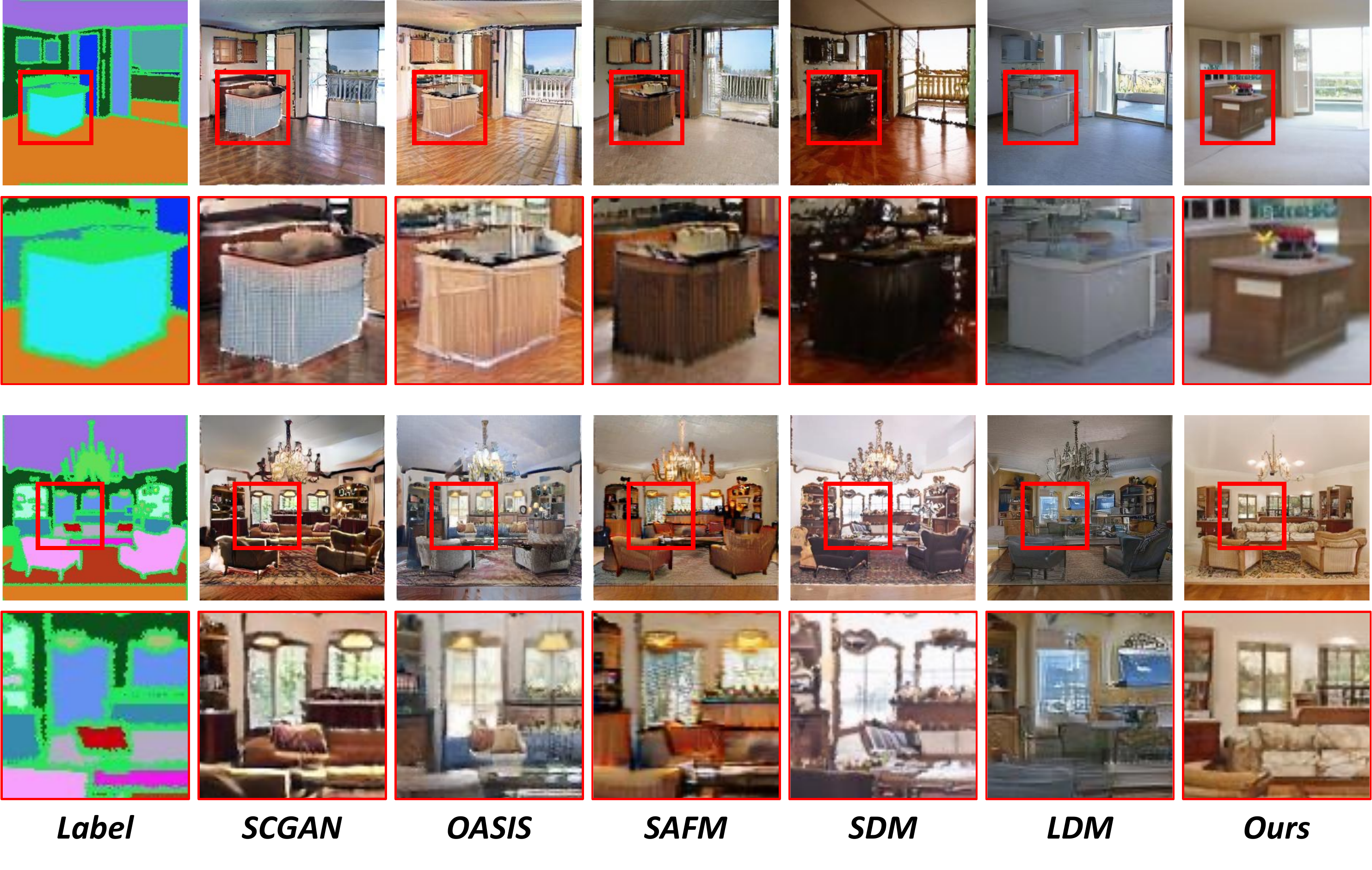}
\caption{ \textbf{More qualitative comparisons on SIS with noisy labels (Edge).} The generation results are conditioned with incomplete masks on the edges of instances.
}
\label{fig:supp-qual-noisy-edge}
\end{figure*}
\begin{figure*}[t!]
\centering
\includegraphics[width=0.85\textwidth]{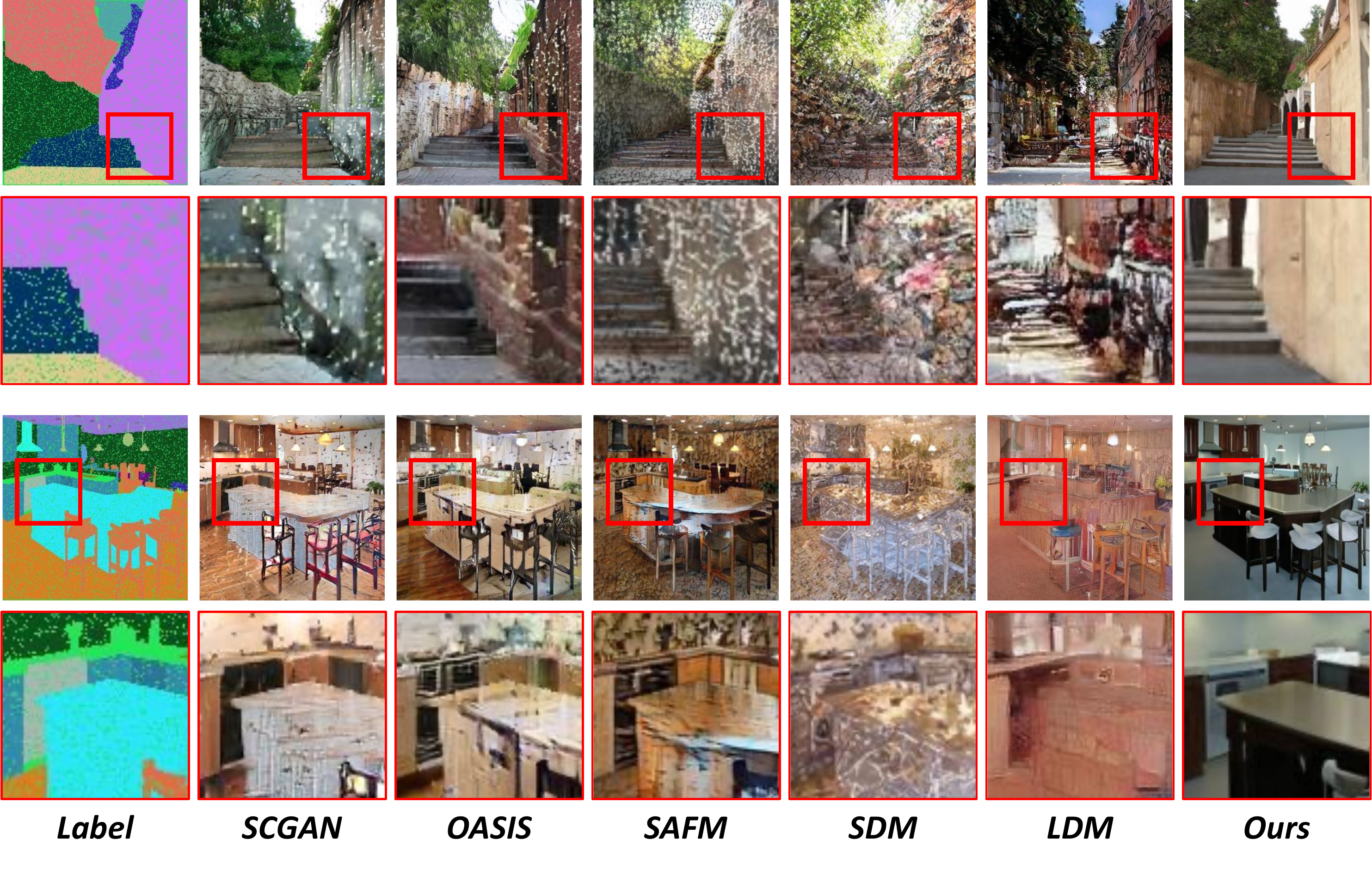}
\caption{ \textbf{More qualitative comparisons on SIS with noisy labels (Random).} The generation results are conditioned with corrupted masks.
}
\label{fig:supp-qual-noisy-rand}
\end{figure*}
\begin{figure*}[ht!]
\centering
\includegraphics[width=0.95\textwidth]{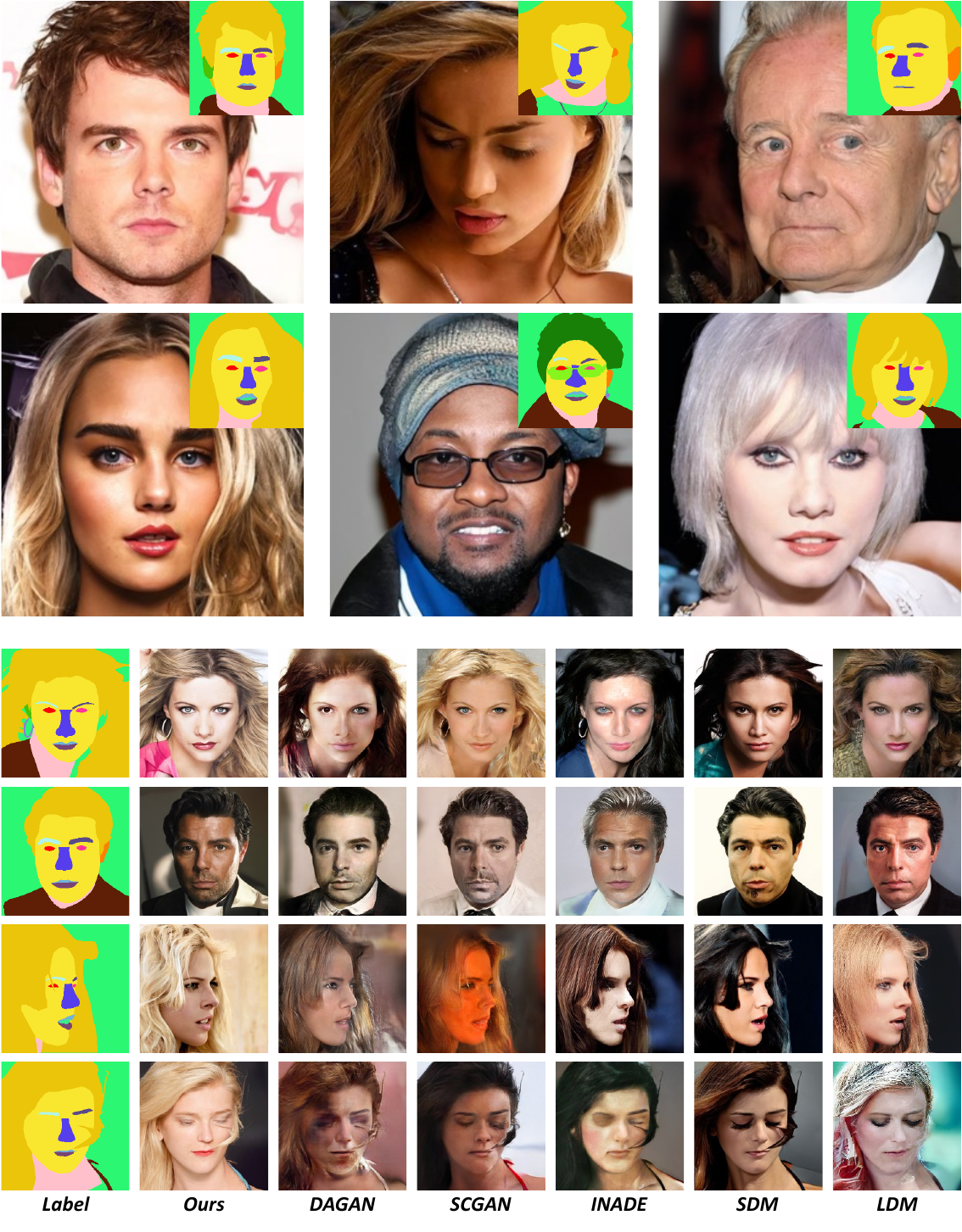}
\caption{ \textbf{More qualitative results and comparisons on \textit{CelebAMask-HQ}.} The first two rows are the result of our model, while the other rows depict qualitative comparisons in terms of generation quality.}
\label{fig:supp_celeba}
\end{figure*}
\begin{figure*}[ht!]
\centering
\includegraphics[width=0.9\textwidth]{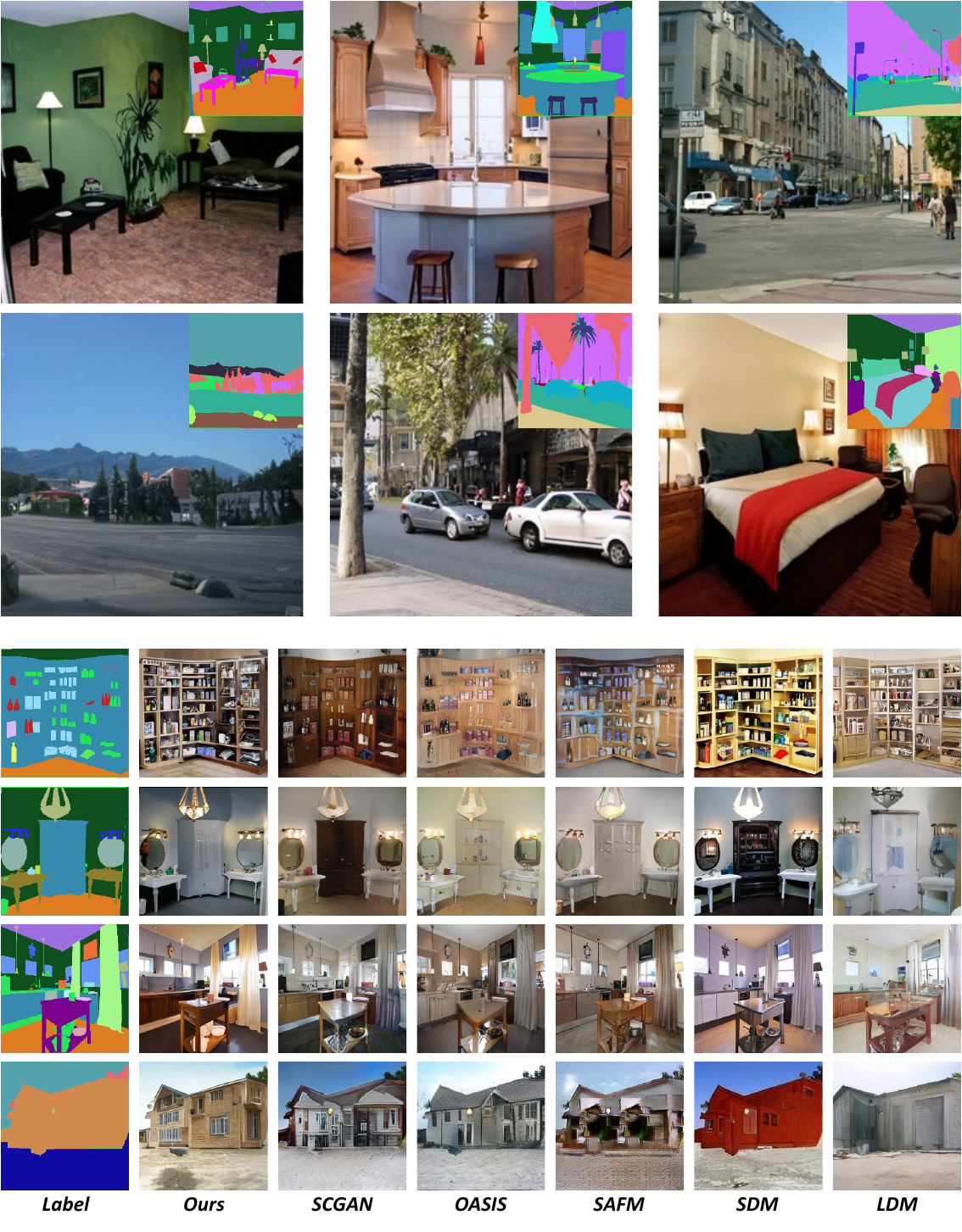}
\caption{ \textbf{More qualitative results and comparisons on \textit{ADE20K}.} The first two rows show the result of our model, while the other rows depict qualitative comparisons in terms of generation quality.
}
\label{fig:supp_ade20k}
\end{figure*}
\begin{figure*}[ht!]
\centering
\includegraphics[width=0.9\textwidth]{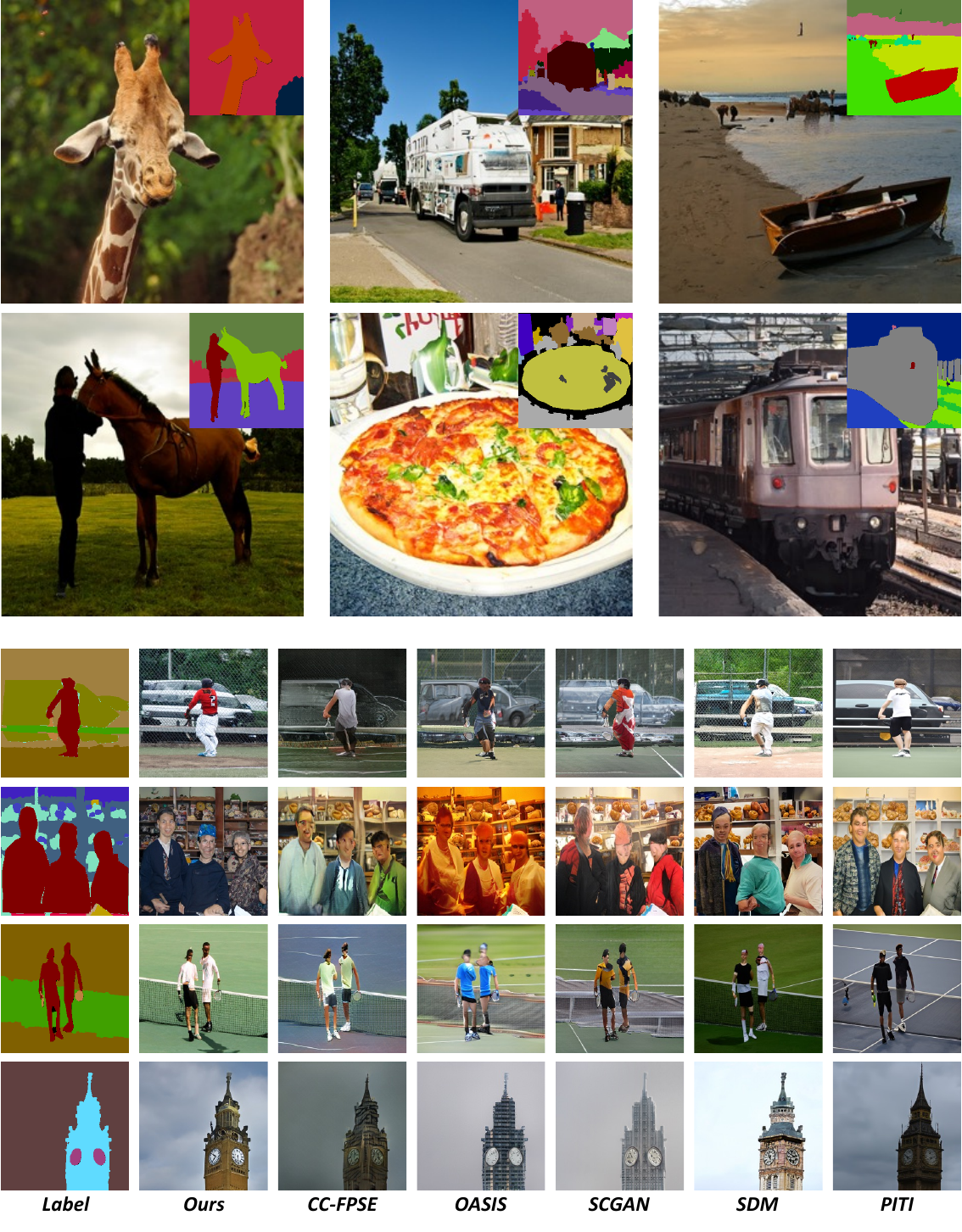}
\caption{ \textbf{More qualitative results and comparisons on \textit{COCO-Stuff}.} The first two rows show the result of our model, while the other rows depict qualitative comparisons in terms of generation quality.
}
\label{fig:supp_coco}
\end{figure*}
% \clearpage
%%%%%%%%%%%%%%%%%%%%%%%%%%%%%%%%%%%%%%%%%%%%%%%%%%%%%%%%%%%%%%%%%%%%%%%%%%%%%%%
%%%%%%%%%%%%%%%%%%%%%%%%%%%%%%%%%%%%%%%%%%%%%%%%%%%%%%%%%%%%%%%%%%%%%%%%%%%%%%%

\end{document}